\def\year{2020}\relax
\documentclass[letterpaper]{article} 
\usepackage{aaai20}  
\usepackage{times}  
\usepackage{helvet} 
\usepackage{courier}  
\usepackage[hyphens]{url}  
\usepackage{graphicx} 
\usepackage{graphicx}  
\usepackage{booktabs} 
\usepackage{amsmath}
\usepackage{amssymb}
\usepackage{dsfont}
\usepackage{xspace} 
\usepackage{amsthm}  
\usepackage{thmtools}
\usepackage{thm-restate}
\usepackage{etoolbox}
\usepackage{bm}
\usepackage{amsfonts}
\usepackage[capitalise, noabbrev]{cleveref}
\usepackage{pifont}
\usepackage{fixltx2e}
\usepackage{mathtools}
\usepackage{mathabx}
\usepackage{algorithm}
\usepackage{algorithmicx,algpseudocode}
\usepackage{mathtools}
\usepackage{subcaption}
\usepackage{mathrsfs}
\usepackage[percent]{overpic}

\declaretheorem[name=Theorem,numberwithin=section]{thr}
\declaretheorem[name=Definition,numberwithin=section]{defi}
\declaretheorem[name=Proposition,numberwithin=section]{prop}

\declaretheorem[name=Lemma,numberwithin=section]{lemma}
\declaretheorem[name=Identification Rule,numberwithin=section]{test}

\newcommand{\Renyi}{R\'{e}nyi }
\DeclareRobustCommand{\eg}{e.g.,\@\xspace}                         
\DeclareRobustCommand{\ie}{i.e.,\@\xspace}                         
\DeclareRobustCommand{\wrt}{w.r.t.\@\xspace}                       

\newcommand{\mathbr}[1]{\bm{\mathbf{#1}}}
\DeclareRobustCommand{\quotes}[1]{``#1''}

\DeclareMathOperator*{\E}{\mathbb{E}}

\DeclareMathOperator*{\Cov}{\mathbb{C}\mathrm{ov}}

\DeclareMathOperator*{\argmin}{arg\,min}
\DeclareMathOperator*{\argmax}{arg\,max}

\DeclareMathOperator*{\argsup}{arg\,sup}

\newcommand{\interval}{\{1,...,d\}}
\newcommand{\BigO}{\mathcal{O}}

\newcommand{\citet}[1]{\citeauthor{#1} \shortcite{#1}}
\newcommand{\citep}{\cite}
\newcommand{\citealp}[1]{\citeauthor{#1} \citeyear{#1}}

\usepackage[colorinlistoftodos, textwidth=40mm, shadow]{todonotes}

\AtBeginEnvironment{proof}{\small}
\allowdisplaybreaks[4]

\title{Policy Space Identification in Configurable Environments}
\author{
Alberto Maria Metelli, Guglielmo Manneschi, Marcello Restelli\\
Dipartimento di Elettronica, Informazione e Bioingegneria\\
Politecnico di Milano\\
Piazza Leonardo da Vinci, 32, 20133, Milano, Italy\\
\small {albertomaria.metelli@polimi.it}, {guglielmo.manneschi@mail.polimi.it}, {marcello.restelli@polimi.it}
}

\begin{document}
\setlength{\abovedisplayskip}{2pt}
\setlength{\belowdisplayskip}{2pt}

\maketitle
\begin{abstract}
\begin{quote}
We study the problem of identifying the policy space of a learning agent, having access to a set of demonstrations generated by its optimal policy. We introduce an approach based on statistical testing to identify the set of policy parameters the agent can control, within a larger parametric policy space. After presenting two identification rules (combinatorial and simplified), applicable under different assumptions on the policy space, we provide a probabilistic analysis of the simplified one in the case of linear policies belonging to the exponential family. To improve the performance of our identification rules, we frame the problem in the recently introduced framework of the Configurable Markov Decision Processes, exploiting the opportunity of configuring the environment to induce the agent revealing which parameters it can control. Finally, we provide an empirical evaluation, on both discrete and continuous domains, to prove the effectiveness of our identification rules.
\end{quote}
\end{abstract}

\section{Introduction}\label{sec:intro}
Reinforcement Learning (RL,~\citealp{sutton2018reinforcement}) deals with sequential decision--making problems in which an
artificial agent interacts with an environment by sensing \emph{perceptions} and performing \emph{actions}. The agent's goal is to find an optimal policy, \ie a prescription of actions that maximizes the (possibly discounted) cumulative reward collected during its interaction with the environment. The performance of an agent in 
an environment is constrained by its perception and actuation possibilities, along with the ability in \emph{mapping} observations to actions. These three elements define the agent's \emph{policy space}. Agents with different policy spaces could display different optimal behaviors, even in the same environment. Therefore, the notion of optimality is necessarily connected to the agent's policy space. While in tabular RL we typically assume access to the full (and finite) space of Markovian stationary policies, in continuous control, the policy space needs to be limited. In policy search methods~\cite{deisenroth2013survey}, the policy is explicitly modeled considering a parametric functional space~\cite{sutton1999policy,peters2008reinforcement} or a kernel space~\cite{deisenroth2011pilco,levine2013guided}; but also in value--based RL, a function approximator induces a set of representable (greedy) policies. 

The knowledge of the agent's policy space could be of crucial importance when the learning process involves the presence of an external supervisor. Recently, the notion of Configurable Markov Decision Process (Conf--MDP,~\citealp{metelli2018configurable}) has been introduced to account for the real--world scenarios in which it is possible to exercise a, maybe partial, control over the environment, by means of a set of environmental parameters (\eg \citealp{silva2018what}; \citealp{silva2019atheoretical}). This activity, called \emph{environment configuration}, can be carried out by the agent itself or by an external supervisor. While previous works focused on the former case (\eg~\citealp{metelli2018reinforcement}), in this paper, we explicitly consider the presence of a supervisor who acts on the environment with the goal of finding the most suitable configuration for the agent. 

Intuitively, the best environment configuration is intimately related to the possibilities of the agent in terms of policy space. For instance, in a car racing problem, the best car configuration depends on the car driver and has to be selected, by a track engineer (the supervisor), according to the driver's skills. Thus, the external supervisor has to be aware of the agent's policy space. Besides the Conf--MDPs, there are other contexts in which knowing the policy space can be beneficial, such as Imitation Learning, \ie the framework in which an agent learns by observing an expert~\cite{osa2018analgorithmic}. In behavioral cloning, where recovering an imitating policy is cast as a supervised learning problem~\cite{argall2009asurvey}, knowing the expert's policy space means knowing a suitable hypothesis space, preventing possible over/underfitting phenomena. However, also Inverse Reinforcement Learning algorithms (IRL,~\citealp{ng2000algorithms}), whose goal is to retrieve a reward function explaining the expert's behavior, can gain some advantages. In particular, the IRL approaches based on the policy gradient (\eg~\citealp{pirotta2016inverse}; \citealp{metelli2017compatible}; \citealp{tateo2017gradient}) require a parametric representation of the expert's policy, whose choice might affect the quality of the recovered reward function.
%

In this paper, motivated by the examples presented above, we study the problem of \emph{identifying} the agent's policy space in a Conf--MDP, by observing the agent's behavior and, possibly, exploiting the \emph{configuration} opportunities of the environment. We consider the case in which the policy space of the agent is a subset of a known super--policy space $\Pi_{\Theta}$ induced by a parameter space $\Theta \subseteq \mathbb{R}^d$. Thus, any policy $\pi_{\mathbr{\theta}}$ is determined by a $d$--dimensional parameter vector $\mathbr{\theta} \in \Theta$. However, the agent has control over a smaller number $d^* < d$ of parameters (which are unknown), while the remaining ones have a fixed value, namely zero.\footnote{The choice of zero, as fixed value, might appear arbitrary, but it is rather a common case in practice. For instance, in a linear policy the fact that the agent does not observe a state feature is equivalent to set the corresponding parameters to zero. In a neural network, removing a neuron is equivalent to neglect all its connections, which in turn can be realized by setting the relative weights to zero.} Our goal is to identify the parameters that the agent can control (and possibly change) by observing demonstrations of the optimal policy $\pi^*$. It is worth noting that the formulation based on the identification of the \emph{parameters} effectively covers the limitations of the policy space related to perception, actuation, and mapping.
To this end, we formulate the problem as deciding whether each parameter $\theta_i$ for $i \in \interval$ is zero, and we address it by means of a statistical test. In other words, we check whether there is a statistically significant difference between the likelihood of the agent's behavior with the full set of parameters and the one in which $\theta_i$ is set to zero. In such case, we conclude that $\theta_i$ is not zero and, consequently, the agent can control it. On the contrary, either the agent cannot control the parameter or zero is the value consciously chosen by the agent. 

Indeed, there could be parameters that, given the peculiarities of the environment, are useless for achieving the optimal behavior or whose optimal value is actually zero, while they could prove to be essential in a different environment. For instance, in a grid world where the goal is to reach the right edge, the vertical position of the agent is useless, while if the goal is to reach the upper right corner both horizontal and vertical positions become relevant.
In this spirit, configuring the environment can help the supervisor in identifying whether a parameter set to zero is actually uncontrollable by the agent or just useless in the current environment. Thus, the supervisor can change the environment configuration $\mathbr{\omega} \in \Omega$, so that the agent will adjust its policy, possibly changing the parameter value and revealing whether it can control such parameter. Thus, the new configuration should induce an optimal policy in which the considered parameters have a value significantly different from zero. We formalize this notion as the problem of finding the new environment configuration that maximizes the \emph{power} of the statistical test and we propose a surrogate objective for this purpose. 

It is worth emphasizing that we use the Conf--MDP notion for two purposes. First, we propose the problem of learning the optimal configuration in a Conf--MDP as a motivating example in which the knowledge of the policy space is valuable. Second, we use the environment configurability as a tool to improve the identification of the policy space. 

The paper is organized as follows. In Section~\ref{sec:preliminaries}, we introduce the necessary background. The \emph{identification rules} to perform parameter identification in a fixed environment are presented in Section~\ref{sec:fixedEnv} and analyzed in Section~\ref{sec:analysis}. Section~\ref{sec:confEnv} shows how to improve them by exploiting the environment configurability. Finally, the experimental evaluation, on discrete and continuous domains, is provided in Section~\ref{sec:experiments}. The proofs of all the results can be found in Appendix~\ref{apx:proofs}.

%
%

\section{Preliminaries}\label{sec:preliminaries}
In this section, we report the essential background that will be used in the subsequent sections.

\paragraph{(Configurable) Markov Decision Processes} A discrete--time Markov Decision Process (MDP,~\citealp{puterman2014markov}) is defined by the tuple $\mathcal{M} = \left(\mathcal{S}, \mathcal{A}, p, \mu, r, \gamma \right)$, where $\mathcal{S}$ and $\mathcal{A}$ are the state space and the action space respectively, $p$ is the transition model that provides, for every state-action pair $(s,a) \in \mathcal{S \times A}$, a probability distribution over the next state $p(\cdot|s,a)$, $\mu$ is the distribution of the initial state, $r$ is the reward model defining the reward collected by the agent $r(s,a)$ when performing action $a\in \mathcal{A}$ in state $s\in\mathcal{S}$, and $\gamma \in [0,1]$ is the discount factor. The behavior of an agent is defined by means of a policy $\pi$ that provides a probability distribution over the actions $\pi(\cdot|s)$ for every state $s \in \mathcal{S}$. An MDP $\mathcal{M}$ paired with a policy $\pi$ induces a $\gamma$--discounted stationary distribution over the states~\cite{sutton1999policy}, defined as $d_{\mu}^{\pi}(s) = (1-\gamma) \sum_{t=0}^{+\infty} \gamma^t \Pr \left(s_t = s | \mathcal{M},\,\pi \right)$. We limit the scope to parametric policy spaces $\Pi_\Theta = \left\{ \pi_{\mathbr{\theta}} : \mathbr{\theta} \in \Theta \right\}$, where $\Theta \subseteq \mathbb{R}^d$ is the parameter space. The goal of the agent is to find an optimal policy, \ie any policy parametrization that maximizes the \emph{expected return}: 
\begin{equation} 
	\mathbr{\theta}^* \in \argmax_{\mathbr{\theta} \in \Theta} J_{\mathcal{M}}(\mathbr{\theta}) = \frac{1}{1-\gamma}\E_{\substack{s \sim d_{\mu}^{\pi_{\mathbr{\theta}}} \\ a \sim \pi_{\mathbr{\theta}} (\cdot|s)}} \left[r(s,a) \right].
\end{equation}
In this paper, we consider a slightly modified version of the Conf--MDPs~\cite{metelli2018configurable}.
\begin{defi}
A Configurable Markov Decision Process (Conf--MDP) induced by the configuration space $\Omega \subseteq \mathbb{R}^p$ is defined as the set of MDPs $\mathcal{C}_{\Omega} = \left\{ \mathcal{M}_{\mathbr{\omega}} = \left( \mathcal{S}, \mathcal{A}, p_{\mathbr{\omega}}, \mu_{\mathbr{\omega}}, r, \gamma \right)\,:\, \mathbr{\omega} \in \Omega \right\}$.
\end{defi}
The main differences \wrt the original definition are: i) we allow the configuration of the initial state distribution $\mu_{\mathbr{\omega}}$, in addition to the transition model $p_{\mathbr{\omega}}$; ii) we restrict to the case of parametric configuration spaces $\Omega$; iii) we do not consider the policy space $\Pi_\Theta$ as a part of the Conf--MDP.  


\paragraph{Generalized Likelihood Ratio Test}
The Generalized Likelihood Ratio test~(GLR, \citealp{barnard1959control}; \citealp{casella2002statistical}) aims at testing the goodness of fit of two statistical models. Given a parametric model having density function $p_{\mathbr{\theta}}$ with $\mathbr{\theta} \in \Theta$, we aim at testing the null hypothesis $\mathcal{H}_0 : \mathbr{\theta}^* \in \Theta_0$, where $\Theta_0 \subset \Theta$ is a subset of the parametric space, against the alternative $\mathcal{H}_1 : \mathbr{\theta}^* \in \Theta \setminus \Theta_0$. Given a dataset $\mathcal{D} = \left\{ X_i \right\}_{i=1}^n$ sampled independently from $p_{\mathbr{\theta}^*}$, where $\mathbr{\theta}^*$ is the true parameter, the GLR statistic is:
\begin{equation}
	\Lambda = \frac{\sup_{\mathbr{\theta} \in \Theta_0} \widehat{\mathcal{L}}(\mathbr{\theta}) }{\sup_{\mathbr{\theta} \in \Theta}  \widehat{\mathcal{L}}(\mathbr{\theta})},
\end{equation}
where $\widehat{\mathcal{L}}(\mathbr{\theta}) = \prod_{i=1}^n p_{\mathbr{\theta}}(X_i)$ is the likelihood function. We denote with $\widehat{\ell}(\mathbr{\theta}) = -\log \widehat{\mathcal{L}}(\mathbr{\theta})$ the negative log--likelihood function, $\widehat{\mathbr{\theta}} \in \argsup_{\mathbr{\theta} \in \Theta} \widehat{\mathcal{L}}(\mathbr{\theta}) $ and $\widehat{\mathbr{\theta}}_0 \in \argsup_{\mathbr{\theta} \in \Theta_0}  \widehat{\mathcal{L}}(\mathbr{\theta})$, \ie the maximum likelihood solutions in $\Theta$ and $\Theta_0$ respectively. Moreover, we define the expectation of the likelihood under the true parameter: $\ell(\mathbr{\theta}) = \E_{X_i \sim p_{\mathbr{\theta}^*}} [\widehat{\ell}(\mathbr{\theta})]$. As the maximization is carried out employing the same dataset $\mathcal{D}$ and recalling that $\Theta_0 \subset \Theta$, we have that $\Lambda \in [0,1]$. 
It is usually convenient to consider the logarithm of the GLR statistic: $\lambda = -2 \log \Lambda = 2 (\widehat{\ell}(\widehat{\mathbr{\theta}}_0) - \widehat{\ell}(\widehat{\mathbr{\theta}}) )$. Therefore, $\mathcal{H}_0$ is rejected for large values of $\lambda$, \ie when the maximim likelihood parameter searched in the restricted set $\Theta_0$ significantly underfits the data $\mathcal{D}$, compared to $\Theta$. Wilk's theorem provides the asymptomatic distribution of $\lambda$ when $\mathcal{H}_0$ is true~\cite{wilks1938large,casella2002statistical}.

\begin{restatable}[\citealp{casella2002statistical}, Theorem 10.3.3]{thr}{}\label{thr:wilks}
	Let $d = \mathrm{dim}(\Theta)$ and $d_0 = \mathrm{dim}(\Theta_0) < d$. Under suitable regularity conditions (see~\citet{casella2002statistical} 10.6.2), if $\mathcal{H}_0$ is true, then when $n \rightarrow +\infty$, the distribution of $\lambda$ tends to a $\chi^2$ distribution with $d-d_0$ degrees of freedom.
\end{restatable}

The \emph{significance} of a test $\alpha \in [0,1]$, or \emph{type I error} probability, is the probability to reject $\mathcal{H}_0$ when $\mathcal{H}_0$ is true, while the \emph{power} of a test $1-\beta \in [0,1]$ is the probability to reject $\mathcal{H}_0$ when $\mathcal{H}_0$ is false, $\beta$ is the \emph{type II error} probability.

\section{Policy Space Identification in a Fixed Environment}\label{sec:fixedEnv}
As we introduced in Section~\ref{sec:intro}, we aim at identifying the agent's policy space, by observing a set of demonstrations coming from the optimal policy $\pi^* \in \Pi_{\Theta}$\footnote{We do not explicitly report the dependence on the agent's parameter $\mathbr{\theta}^* \in \Theta$ as, in the general case, there might exist multiple parameters yielding the same policy $\pi^*$.} only, \ie $\mathcal{D} = \{(s_i,a_i)\}_{i=1}^n$ where $s_i \sim d_{\mu}^{\pi^*}$ and $a_i \sim \pi^*(\cdot|s_i)$ sampled independently. In particular, we assume that the agent has control over a limited number of parameters $d^* < d$ whose value can be changed during learning, while the remaining $d-d^*$ are kept fixed to zero.\footnote{The extension of the identification rules to (known) fixed values different from zero is straightforward.} Given a set of indexes $I \subseteq \interval$ we define the subset of the parameter space:  $\Theta_I = \left\{ \mathbr{\theta} \in \Theta : \theta_i = 0,\, \forall i \in I \setminus \interval \right\}$. Thus, the set $I$ represents the indexes of the parameters that can be changed if the agent's parameter space were $\Theta_I$. Our goal is to find a set of parameter indexes $I^*$ that are \emph{sufficient} to explain the agent's policy, \ie $\pi^* \in \Pi_{\Theta_{I^*}}$ but also \emph{necessary}, in the sense that when removing any $i \in I^*$ the remaining ones are insufficient to explain the agent's policy, \ie $\pi^* \notin\Pi_{\Theta_{I^* \setminus \{i\}}}$. We formalize these notions in the following definition.
\begin{restatable}[Correctness]{defi}{}\label{eq:problemOriginal}
	Let $\pi^* \in \Pi_{\Theta}$. A set of parameter indexes $I^* \subseteq \interval$ is \emph{correct} \wrt $\pi^*$ if:
	\begin{equation*}
	\pi^* \in \Pi_{\Theta_{I^*}} \, \wedge \,\forall i \in {I^*} : \pi^* \notin \Pi_{\Theta_{I^* \setminus \{i\}}}.
	\end{equation*}
	We denote with $\mathcal{I}^*$ the set of all correct $I^*$.
\end{restatable}

The uniqueness of $I^*$ is guaranteed under the assumption that each policy admits a unique representation in $\Pi_\Theta$.
\begin{restatable}[Identifiability]{ass}{}\label{ass:identifiability}
	The policy space $\Pi_{\Theta}$ is \emph{identifiable}, \ie for all $\mathbr{\theta},\mathbr{\theta}' \in \Theta$, we have:
	\begin{equation*}
		\pi_{\mathbr{\theta}} = \pi_{\mathbr{\theta}'}\; \text{almost surely} \implies \mathbr{\theta} = \mathbr{\theta}'.
	\end{equation*}
\end{restatable}

The identifiability property allows rephrasing Definition~\ref{eq:problemOriginal} in terms of the policy parameters only.

\begin{restatable}[]{lemma}{lemmaRuleIdent}\label{lemma:lemmaRuleIdent}
	Under Assumption~\ref{ass:identifiability}, let $\mathbr{\theta}^* \in \Theta$ be the unique parameter such that $\pi_{\mathbr{\theta}^*} = \pi^*$ almost surely. Then, there exists a unique set of parameter indexes $I^* \subseteq \interval$ (\ie $\mathcal{I}^* = \{I^*\}$) that is correct \wrt $\pi^*$:
	\begin{equation*}
	I^* = \left\{ i \in \interval \,:\, \theta_i^* \neq 0 \right\}.
	\end{equation*}
\end{restatable}

The following two subsections are devoted to the presentation of the \emph{identification rules} based on the application of Definition~\ref{eq:problemOriginal} (Section~\ref{sec:identif}) and Lemma~\ref{lemma:lemmaRuleIdent} (Section~\ref{sec:identifSimp}) when we only have access to a dataset of samples $\mathcal{D}$. The goal of an identification rule consists in producing a set $\widehat{\mathcal{I}}$, approximating $\mathcal{I}^*$. 

\subsection{Combinatorial Identification Rule}\label{sec:identif}
In principle, using $\mathcal{D} = \{(s_i,a_i)\}_{i=1}^n$, we could compute the maximum likelihood parameter $\widehat{\mathbr{\theta}}\in \argsup_{\mathbr{\theta} \in \Theta} \widehat{\mathcal{L}}(\mathbr{\theta})$ and employ it with Definition~\ref{eq:problemOriginal}. However, this approach has, at least, two drawbacks. First, when Assumption~\ref{ass:identifiability} is not fulfilled, it would produce a single approximate parameter, while multiple choices might be viable. Second, because of the estimation errors, we would hardly get a zero value for the parameters the agent might not control. For this reasons, we employ a GLR test to assess whether a specific set of parameters is zero. Specifically, for all $I \subseteq \interval$ we consider the pair of hypotheses $\mathcal{H}_{0,I} \,:\, \pi^* \in \Pi_{\Theta_I}$ against $\mathcal{H}_{1,I} \,:\, \pi^* \in \Pi_{\Theta \setminus \Theta_I}$ and the GLR statistic:
\begin{equation}\label{eq:statistic}
	\lambda_I  = -2 \log \frac{\sup_{\mathbr{\theta} \in \Theta_I} \widehat{\mathcal{L}}(\mathbr{\theta})  }{\sup_{\mathbr{\theta} \in \Theta} \widehat{\mathcal{L}}(\mathbr{\theta}) } = 2 \left( \widehat{\ell}(\widehat{\ell}(\widehat{\mathbr{\theta}}_I) - \widehat{\mathbr{\theta}})  \right),
\end{equation}
where the likelihood is defined as $\widehat{\mathcal{L}}(\mathbr{\theta}) = \prod_{i=1}^n \pi_{\mathbr{\theta}}(a_i|s_i)$, $\widehat{\mathbr{\theta}}_I \in \argsup_{\mathbr{\theta} \in \Theta_I} \widehat{\mathcal{L}}(\mathbr{\theta})$ and $\widehat{\mathbr{\theta}} \in \argsup_{\mathbr{\theta} \in \Theta} \widehat{\mathcal{L}}(\mathbr{\theta})$. We now state the identification rule derived from Definition~\ref{eq:problemOriginal}.
\begin{test}\label{ir:complete}
$\widehat{\mathcal{I}}_c$ contains all and only the sets of parameter indexes ${I} \subseteq \interval $ such that:
\begin{equation}\label{eq:complete}
	\lambda_{{I}} \le c(|{I}|) \wedge \,\forall i \in {I}  : \lambda_{{I} \setminus \{i\}}> c(|{I}\setminus \{i\}|),
\end{equation}
where $c(l)$ are the \emph{critical values}.
\end{test}
Thus, ${I}$ is defined in such a way that the null hypothesis $\mathcal{H}_{0,{I}}$ is not rejected, \ie ${I}$ contains parameters that are sufficient to explain the data $\mathcal{D}$, and necessary since for all $i \in {I}$ the set ${I} \setminus \{i\}$ is no longer sufficient, as  $\mathcal{H}_{0,{I} \setminus\{i\}}$ is rejected. 
The critical values $c(l)$, that depend on the cardinality $l$ of the tested set of indexes, should be determined in order to enforce guarantees on the type I and II errors. 
We will show in Section~\ref{sec:experiments} how to set them in practice. Refer to Algorithm~\ref{alg:Test} for the pseudocode of the identification rule.

\setlength{\textfloatsep}{10pt}
\begin{algorithm}[tb]
    \caption{Identification Rule~\ref{ir:complete} (Combinatorial)}
    \label{alg:Test}
    \small
    \textbf{input}: dataset $\mathcal{D}$, parameter space $\Theta$, critical values $c$
    \begin{algorithmic} 
        	\State $\widehat{\mathcal{I}}_{c}  \leftarrow \{\}$
            \State $\widehat{\mathcal{L}} = \max_{\mathbr{\theta} \in \Theta}  \widehat{\mathcal{L}}(\mathbr{\theta})$
            \For{$I \subseteq \interval$ in sorted by cardinality}
            \State $\widehat{\mathcal{L}}_{I} = \max_{\mathbr{\theta} \in \Theta_I}  \widehat{\mathcal{L}}(\mathbr{\theta})$
            \State $\lambda_I = -2 \log \frac{\widehat{\mathcal{L}}_{I}}{\widehat{\mathcal{L}}}$
			\If{$\lambda_I \le c(|I|)$ \textbf{and} $\forall i \in I \,:\, \lambda_{I\setminus \{i\}} > c(|I\setminus \{i\}|)$}
            		\State $\widehat{\mathcal{I}}_{c}  \leftarrow \widehat{\mathcal{I}}_{c} \cup \{ I \}$
            \EndIf    
            \EndFor
            \State \textbf{return} $\widehat{\mathcal{I}}_{c}$
    \end{algorithmic}
\end{algorithm}

\subsection{Simplified Identification Rule}\label{sec:identifSimp}
Identification Rule~\ref{ir:complete} is usually impractical, as it requires performing $\BigO\left( 2^d \right)$ statistical tests. However, under Assumption~\ref{ass:identifiability}, to retrieve $I^*$ we do not need to test all subsets, but we can just examine one parameter at a time (see Lemma~\ref{lemma:lemmaRuleIdent}). Thus, for all $i \in \interval$ we consider the pair of hypotheses $\mathcal{H}_{0,i} \,:\, {\theta}^*_i = 0$ against $\mathcal{H}_{1,i} \,:\, {\theta}^*_i \neq 0$ and define $\Theta_i = \{ \mathbr{\theta} \in \Theta\,:\, \theta_i = 0\}$. The GLR test can be performed straightforwardly, using the statistic:
\begin{equation}\label{eq:statSimplified}
	\lambda_i =  -2 \log \frac{\sup_{\mathbr{\theta} \in \Theta_i} \widehat{\mathcal{L}}(\mathbr{\theta})  }{\sup_{\mathbr{\theta} \in \Theta} \widehat{\mathcal{L}}(\mathbr{\theta}) } = 2 \left(\widehat{\ell}(\widehat{\mathbr{\theta}}_i) - \widehat{\ell}(\widehat{\mathbr{\theta}}) \right),
\end{equation}
where the likelihood is defined as $\widehat{\mathcal{L}}(\mathbr{\theta}) = \prod_{i=1}^n \pi_{\mathbr{\theta}}(a_i|s_i)$, $\widehat{\mathbr{\theta}}_i = \argsup_{\mathbr{\theta} \in \Theta_i} \widehat{\mathcal{L}}(\mathbr{\theta})$ and $\widehat{\mathbr{\theta}} = \argsup_{\mathbr{\theta} \in \Theta} \widehat{\mathcal{L}}(\mathbr{\theta})$.\footnote{This setting is equivalent to a particular case the combinatorial rule in which $\mathcal{H}_{\star,i} \equiv \mathcal{H}_{\star,\interval \setminus \{i\} }$, with $\star \in \{0,1\}$ and, consequently, $\lambda_i \equiv \lambda_{\interval \setminus \{i\}}$ and $\Theta_i = \Theta_{\interval \setminus \{i\}}$.} In the spirit of Lemma~\ref{lemma:lemmaRuleIdent}, we define the identification rule.

\begin{test}\label{ir:simplified}
$\widehat{\mathcal{I}}_c$ contains the unique set of parameter indexes $\widehat{I}_{c} $ such that:
\begin{equation}\label{eq:simplified}
	\widehat{I}_{c} = \left\{i \in \interval : \lambda_i > c(1) \right\},
\end{equation}
where $c(1)$ is the \emph{critical value}.
\end{test}

Therefore, the identification rule constructs $\widehat{I}_c$ by taking all the indexes $i \in \interval$ such that the corresponding null hypothesis $\mathcal{H}_{0,i} \,:\, {\theta}^*_i = 0$ is rejected, \ie those for which there is statistical evidence that their value is not zero. We will show in Section~\ref{sec:analysis} how the critical value $c(1)$ can be computed, in a theoretically sound way, for linear policies belonging to the exponential family.

This second procedure requires a test for every parameter, \ie $\BigO(d)$ instead of $\BigO(2^d)$ tests. However, it comes with the cost of assuming the identifiability property. What happens if we employ this second procedure in a case where the assumption does not hold? 
Consider for instance the case in which two parameters are exchangeable, we will include none of them in $\widehat{I}_{c}$ as, individually, they are not necessary to explain the agent's policy. Refer to Algorithm~\ref{alg:LinearTest} for the pseudocode of the identification rule.

\setlength{\textfloatsep}{10pt}
\begin{algorithm}[tb]
    \caption{Identification Rule~\ref{ir:simplified} (Simplified)}
    \label{alg:LinearTest}
    \small
    \textbf{input}: dataset $\mathcal{D}$, parameter space $\Theta$, critical value $c$
    \begin{algorithmic} 
        	\State $\widehat{I}_{c}  \leftarrow \{\}$
            \State $\widehat{\mathcal{L}} = \max_{\mathbr{\theta} \in \Theta}  \widehat{\mathcal{L}}(\mathbr{\theta})$
            \For{$i \in \interval$}
				\State $\widehat{\mathcal{L}}_{i} = \max_{\mathbr{\theta} \in \Theta_i}  \widehat{\mathcal{L}}(\mathbr{\theta})$
				\State $\lambda_i = -2 \log \frac{\widehat{\mathcal{L}}_{i}}{\widehat{\mathcal{L}}}$
				\If{$\lambda_i > c(1)$}
					\State $\widehat{I}_{c} \leftarrow \widehat{I}_{c} \cup \{i\}$
				\EndIf
            \EndFor
            \State \textbf{return} $\{\widehat{I}_{c}\}$
    \end{algorithmic}
\end{algorithm}

\begin{table*}
\small
\centering
	\begin{tabular}{lccccc} 
		\toprule
		Policy & $\mathcal{A}$ & $\pi_{\widetilde{\mathbr{\theta}}}$  & $\mathbr{t}$ & $h$ \\
		\midrule
		Gaussian & $\mathbb{R}^k$ & $ \pi_{\widetilde{\mathbr{\theta}}}(\mathbr{a}|s) = \frac{\exp\left\{ -\frac{1}{2} (\mathbr{a} - \widetilde{\mathbr{\theta}}\mathbr{\phi}(s))^T\mathbr{\Sigma}^{-1}(\mathbr{a} - \widetilde{\mathbr{\theta}}\mathbr{\phi}(s)) \right\}}{(2\pi)^{\frac{k}{2}} \det(\mathbr{\Sigma})^{\frac{1}{2}}}  $ & $\mathbr{\Sigma}^{-1} \mathbr{a} \otimes \mathbr{\phi}(s)$ & $\displaystyle \frac{\exp \left\{ -\frac{1}{2} \mathbr{a}^T \mathbr{\Sigma}^{-1} \mathbr{a} \right\}}{(2\pi)^{\frac{k}{2}} \det \left(\mathbr{\Sigma} \right)^{\frac{1}{2}}} $\\
		Boltzmann & $\{a_1,...,a_{k+1}\}$ & $\displaystyle \pi_{\widetilde{\mathbr{\theta}}}(a_i|s) = \begin{cases} \frac{e^{\widetilde{\mathbr{\theta}}_i^T \mathbr{\phi}(s)}}{1+\sum_{j=1}^{k} e^{\widetilde{\mathbr{\theta}}_j^T \mathbr{\phi}(s) }}  & \text{if } i \le k  \\ \frac{1}{1+\sum_{j=1}^{k} e^{\widetilde{\mathbr{\theta}}_j^T \mathbr{\phi}(s) }} & \text{if } i = k \end{cases}$ & $ \begin{cases} \mathbr{e}_i \otimes \mathbr{\phi}(s)   & \text{if } i \le k \\ \mathbr{0}  & \text{if } i=k+1 \end{cases}$ & $1$\\
		\bottomrule
\end{tabular}
\caption{Action space $\mathcal{A}$, probability density function $\pi_{\widetilde{\mathbr{\theta}}}$, sufficient statistic $\mathbr{t}$, and function $h$ for the Gaussian linear policy with fixed covariance and the Boltzmann linear policy. For convenience of representation $\widetilde{\mathbr{\theta}} \in \mathbb{R}^{k \times q}$ is a matrix and $\mathbr{\theta} = \mathrm{vec} (\widetilde{\mathbr{\theta}}^T) \in \mathbb{R}^{d}$, with $d=kq$. We denote with $\mathbr{e}_i$ the $i$--th vector of the canonical basis of $\mathbb{R}^k$ and with $\otimes$ the Kronecker product.}\label{tab:expFamily}
\end{table*}

\section{Analysis for the Exponential Family}\label{sec:analysis}
In this section, we provide an analysis of the Identification Rule~\ref{ir:simplified} for a policy $\pi_{\mathbr{\theta}}$ linear in some state features $\mathbr{\phi}$ that belongs to the exponential family~\cite{brown1986fundamentals}.\footnote{We limit our analysis to Identification Rule~\ref{ir:simplified} since we will show that, in the case of linear policies belonging to the exponential family, the identifiability property can be easily enforced.}

\begin{restatable}[Exponential Family]{defi}{}\label{defi:expFamily}
Let $\mathbr{\phi}: \mathcal{S} \rightarrow \mathbb{R}^q$ be a feature function. The policy space $\Pi_{\Theta}$ is a space of linear policies, belonging to the exponential family, if $\Theta = \mathbb{R}^d$ and all policies $\pi_{\mathbr{\theta}} \in \Pi_{\Theta}$ have form:
\begin{equation}\label{eq:expFamily}
	\pi_{\mathbr{\theta}}(a|s) = h(a) \exp \left\{\mathbr{\theta}^T \mathbr{t}\left(s,a \right) - A(\mathbr{\theta},s) \right\},
\end{equation}
where $h$ is a positive function, $\mathbr{t}\left(s,a\right)$ is the \emph{sufficient statistic} that depends on the state via the feature function $\mathbr{\phi}$ (\ie $\mathbr{t}\left(s,a\right) =\mathbr{t}(\mathbr{\phi}(s),a)$) and $A(\mathbr{\theta},s) = \log \int_{\mathcal{A}} h(a)  \exp \{ \mathbr{\theta}^T \mathbr{t}(s,a) \}\mathrm{d} a$ is the log partition function. We denote with $\mathbr{\overline{t}}(s,a,\mathbr{\theta}) =  \mathbr{t}(s,a) - \E_{\overline{a} \sim \pi_{\mathbr{\theta}} (\cdot|s)} \left[ \mathbr{t}(s,\overline{a}) \right]$ the centered sufficient statistic. 
\end{restatable}

This definition allows modelling the linear policies that are often used in RL~\cite{deisenroth2013survey}. Table~\ref{tab:expFamily} shows how to map the Gaussian linear policy with fixed covariance, typically used in continuous action spaces, and the Boltzmann linear policy, suitable for finite action spaces, to Definition~\ref{defi:expFamily} (details in Appendix~\ref{apx:expFamily}). 

For the sake of the analysis, we enforce the following assumption concerning the tail behavior of the policy $\pi_{\mathbr{\theta}}$.
\begin{restatable}[Subgaussianity]{ass}{}\label{ass:subGauss}
	For any $\mathbr{\theta} \in \Theta$ and for any $s \in \mathcal{S}$ the centered sufficient statistic $\mathbr{\overline{t}}(s,a,\mathbr{\theta})$ is subgaussian with parameter $\sigma \ge 0$, \ie for any $\mathbr{\alpha} \in \mathbb{R}^d$:
	\begin{equation*}
		\E_{a \sim \pi_{\mathbr{\theta}} (\cdot|s)} \left[ \exp \left\{ \mathbr{\alpha}^T \mathbr{\overline{t}}(s,a,\mathbr{\theta}) \right\} \right] \le \exp \left\{ \frac{1}{2} \left\| \mathbr{\alpha} \right\|_2^2 \sigma^2 \right\}.
	\end{equation*}
\end{restatable}
Proposition~\ref{prop:subGaussExample} of Appendix~\ref{apx:subgaussian} proves that, when the features are uniformly bounded, \ie $\left\| \mathbr{\phi}(s) \right\|_2 \le \Phi_{\max}$ for all $s \in \mathcal{S}$, this assumption is fulfilled by both Gaussian and Boltzmann linear policies with parameter $\sigma = 2 \Phi_{\max}$ and $\sigma = \Phi_{\max} / \sqrt{ \lambda_{\min} (\mathbr{\Sigma})}$ respectively.

Furthermore, limited to the policies complying with Definition~\ref{defi:expFamily}, the identifiability (Assumption~\ref{ass:identifiability}) can be restated in terms of the Fisher Information matrix~\cite{rothenberg1971identification,little2010parameter}. 

\begin{restatable}[\citealp{rothenberg1971identification}, Theorem 3]{lemma}{identExp}\label{lemma:identExp}
	Let $\Pi_\Theta$ be a policy space, as in Definition~\ref{defi:expFamily}. Then, under suitable regularity conditions (see \citealp{rothenberg1971identification}), if the Fisher Information matrix (FIM) $\mathcal{F}(\mathbr{\theta})$:
	\begin{equation}
		\mathcal{F}(\mathbr{\theta}) = \E_{\substack{s \sim d_{\mu}^{\pi^*} \\ a \sim \pi_{\mathbr{\theta}}(\cdot|s)}} \left[ \mathbr{\overline{t}}(s,a,\mathbr{\theta})\mathbr{\overline{t}}(s,a,\mathbr{\theta})^T \right]
	\end{equation}
	 is non--singular for all $\mathbr{\theta} \in \Theta$, then $\Pi_\Theta$ is identifiable. In this case, we denote with $\lambda_{\min} = \inf_{\mathbr{\theta} \in \Theta} \lambda_{\min} \left( \mathcal{F}(\mathbr{\theta}) \right) > 0$.
\end{restatable}

Proposition~\ref{prop:fisherForGaussBoltz} of Appendix~\ref{apx:fisher} shows that a sufficient condition for the identifiability in the case of Gaussian and Boltzmann linear policies is that the second moment matrix of the feature vector $\E_{s \sim d_{\mu}^{\pi^*}} \left[ \mathbr{\phi}(s)\mathbr{\phi}(s)^T \right]$ is non--singular along with the fact that the policy $\pi_{\mathbr{\theta}}$ plays each action with positive probability for the Boltzmann policy.

\paragraph{Concentration Result} We are now ready to present a concentration result, of independent interest, for the parameters and the negative log--likelihood that represents the central tool of our analysis (details and derivation in Appendix~\ref{apx:concentration}).

\begin{restatable}[]{thr}{}
Under Assumption~\ref{ass:identifiability} and Assumption~\ref{ass:subGauss}, let $\mathcal{D} = \{(s_i,a_i)\}_{i=1}^n$ be a dataset of $n>0$ independent samples, where $s_i \sim d_{\mu}^{\pi_{\mathbr{\theta}^*}}$ and $a_i \sim \pi_{\mathbr{\theta}^*}(\cdot|s_i)$. Let $\widehat{\mathbr{\theta}} = \argmin_{\mathbr{\theta} \in \Theta} \widehat{\ell}(\mathbr{\theta})$ and $\mathbr{\theta}^* = \argmin_{\mathbr{\theta} \in \Theta} {\ell}(\mathbr{\theta})$ . If the empirical FIM:
\begin{equation}
	\widehat{\mathcal{F}}(\mathbr{\theta}) = \frac{1}{n} \sum_{i=1}^n \E_{a \sim \pi_{\mathbr{\theta}}(\cdot|s)} \left[\mathbr{\overline{t}}(s,a,\mathbr{\theta})\mathbr{\overline{t}}(s,a,\mathbr{\theta})^T\right]
\end{equation}
has a positive minimum eigenvalue $\widehat{\lambda}_{\min} > 0$ for all $\mathbr{\theta} \in \Theta$, then, for any $\delta \in [0,1]$, with probability at least $1-\delta$:
	\begin{equation*}
		\left\| \widehat{ \mathbr{\theta}} - \mathbr{\theta}^* \right\|_2 \le \frac{\sigma}{\widehat{\lambda}_{\min}} \sqrt{\frac{2d}{n} \log \frac{2d}{\delta}}.
		\end{equation*}
Furthermore, with probability at least $1-\delta$, individually: 
	\begin{align*}
		&\ell(\widehat{\mathbr{\theta}}) - \ell({\mathbr{\theta}}^*) \le \frac{d^2\sigma^4}{\widehat{\lambda}_{\min}^2 n}  \log \frac{2d}{\delta}\\
		&\widehat{\ell}({\mathbr{\theta}}^*) - \widehat{\ell}(\widehat{\mathbr{\theta}})  \le \frac{ d^2\sigma^4}{\widehat{\lambda}_{\min}^2 n}  \log \frac{2d}{\delta}.
	\end{align*}
\end{restatable}

The theorem shows that the $L^2$--norm of the difference between the maximum likelihood parameter $\widehat{\mathbr{\theta}}$ and the true parameter ${\mathbr{\theta}^*}$ concentrates with rate $\BigO(n^{-1/2})$ while the likelihood $\widehat{\ell}$ and its expectation $\ell$ concentrate with faster rate $\BigO(n^{-1})$. 
Note that the result assumes that the empirical FIM $\widehat{\mathcal{F}}(\mathbr{\theta})$ has a strictly positive eigenvalue $\widehat{\lambda}_{\min} > 0$. This condition can be enforced as long as the true Fisher matrix ${\mathcal{F}}(\mathbr{\theta})$ has a positive minimum eigenvalue $\lambda_{\min}$, \ie under identifiability assumption (Lemma~\ref{lemma:identExp}) and given a sufficiently large number of samples. Proposition~\ref{prop:minEigFisher} of Appendix~\ref{apx:fisher} provides the minimum number of samples such that 
with probability at least $1-\delta$ it holds that $\widehat{\lambda}_{\min} > 0$.

\paragraph{Identification Rule Analysis} The goal of the analysis of the identification rule is to find the critical value $c(1)$ so that the following probabilistic requirement is enforced.

\begin{restatable}[$\delta$--correctness]{defi}{deltaCorrect}\label{def:deltaCorrect}
	Let $\delta \in [0,1]$. An identification rule producing $\widehat{{I}}$ is \emph{$\delta$--correct} if: $\Pr \big( \widehat{{I}} \neq {I}^* \big)\le \delta$. 
\end{restatable}

 We denote with $\alpha = \frac{1}{d-d^*} \E \big[ \big| \big\{ i \notin I^* : i \in \widehat{I}_{c} \big\} \big| \big]$ the expected fraction of parameters that the agent does not control selected by the identification rule and with $\beta = \frac{1}{d^*} \E \big[ \big| \big\{ i \in I^* : i \notin \widehat{I}_{c} \big\} \big| \big]$ the expected fraction of parameters that the agent does  control not selected by the identification rule.\footnote{We use the symbols $\alpha$ and $\beta$ to highlight the analogy between these probabilities and the type I and type II error probabilities of a statistical test. We sometimes refer to $\alpha$ as significance and to $1-\beta$ as power of the identification rules.}
We now provide a result that bounds $\alpha$ and $\beta$ and employs them to derive $\delta$--correctness.
\begin{restatable}[]{thr}{sigPowerSimp}\label{thr:sigPowerSimp}
	Let $\widehat{I}_{c}$ be the set of parameter indexes selected by the Identification Rule~\ref{ir:simplified} obtained using $n>0$ i.i.d. samples collected with $\pi_{\mathbr{\theta}^*}$, with $\mathbr{\theta}^* \in \Theta$. Then, under Assumption~\ref{ass:identifiability} and Assumption~\ref{ass:subGauss}, let ${\mathbr{\theta}}_i^* = \argmin_{\mathbr{\theta} \in \Theta_i} \ell(\mathbr{\theta})$ for all $i \in \interval$ and $\nu = \min \left\{1, \frac{\lambda_{\min}}{\sigma^2} \right\}$. If $\widehat{\lambda}_{\min} \ge \frac{\lambda_{\min}}{2\sqrt{2}}$ and $\ell({\mathbr{\theta}}_i^*) - {l}({\mathbr{\theta}^*}) \ge c(1)$, it holds that:
	{
	\begin{align*}
		&\alpha  \le 2d \exp \left\{ -\frac{c(1) {\lambda}_{\min}^2 n}{16d^2 \sigma^4} \right\}\\
		&\beta \le \frac{2d - 1}{d^*} \sum_{i \in I^*} \exp\left\{ - \frac{ \left( {l}({\mathbr{\theta}}_i^*) - {l}({\mathbr{\theta}^*}) - c(1) \right) {\lambda}_{\min} \nu n}{16(d-1)^2 \sigma^2 } \right\}.
	\end{align*}
	}
	Furthermore, the Identification Rule~\ref{ir:simplified} is $\left((d-d^*)\alpha +d^*\beta\right)$--correct.
\end{restatable}

Since $\alpha$ and $\beta$ are functions of $c(1)$, we could, in principle, employ Theorem~\ref{thr:sigPowerSimp} to enforce a value $\delta$, as in Definition~\ref{def:deltaCorrect}, and derive $c(1)$. However, Theorem~\ref{thr:sigPowerSimp} is not very attractive in practice as it holds under an assumption regarding the minimum eigenvalue of the FIM and the corresponding estimate, \ie $\widehat{\lambda}_{\min} \ge \frac{\lambda_{\min}}{2\sqrt{2}}$, that cannot be verified in practice since $\lambda_{\min}$ is unknown. Similarly, the constants $d^*$, ${l}({\mathbr{\theta}}_i^*)$ and ${l}({\mathbr{\theta}^*})$ are typically unknown. We will provide in Section~\ref{sec:experiments} a heuristic for setting $c(1)$. 

\section{Policy Space Identification in a Configurable Environment}\label{sec:confEnv}
The identification rules presented so far are
unable to distinguish between a parameter set to zero because the agent
cannot control it, or because zero is its optimal value. To overcome this issue, we employ the Conf--MDP properties to select
a configuration in which the parameters we want to examine have an optimal value other than zero. Intuitively, if we want to test whether the agent can control parameter $\theta_i$, we should place the agent in an environment $\mathbr{\omega}_i \in \Omega$ where $\theta_i$ is \quotes{maximally important}
for the optimal policy. This intuition is justified by Theorem~\ref{thr:sigPowerSimp}, since to maximize the \emph{power} of the test ($1-\beta$), all other things being equal, we should maximize the
log--likelihood gap ${l}({\mathbr{\theta}_i^*}) - {l}({\mathbr{\theta}^*})$, \ie parameter $\theta_i$ should
be essential to justify the agent's behavior. Let $I \in \interval$ be a set of parameter
indexes we want to test, our ideal goal is to find the environment $\mathbr{\omega}_I$ such that:
\begin{equation}\label{eq:confProblem}
	\mathbr{\omega}_I \in \argmax_{\mathbr{\omega} \in \Omega} \left\{ {l}({\mathbr{\theta}_I^*}(\mathbr{\omega})) - {l}({\mathbr{\theta}^*}(\mathbr{\omega})) \right\},
\end{equation}
where ${\mathbr{\theta}^*}(\mathbr{\omega}) \in \argmax_{\mathbr{\theta} \in \Theta} J_{\mathcal{M}_{\mathbr{\omega}}}(\mathbr{\theta})$ and ${\mathbr{\theta}}_I^*(\mathbr{\omega}) \in \argmax_{\mathbr{\theta} \in \Theta_I} J_{\mathcal{M}_{\mathbr{\omega}}}(\mathbr{\theta})$ are the parameters of the optimal policies 
in the environment $\mathcal{M}_{\mathbr{\omega}}$ in $\Pi_{\Theta}$ and $\Pi_{\Theta_I}$ respectively. Clearly, given the samples $\mathcal{D}$ collected with a single optimal policy $\pi^*(\mathbr{\omega}_0)$ in a single environment $\mathcal{M}_{\mathbr{\omega}_0}$, solving problem~\eqref{eq:confProblem} is hard as it requires performing an off--distribution optimization both on the space of policy parameters and configurations. For these reasons, we consider a surrogate objective that assumes that the optimal parameter in the new configuration can be reached by performing a single gradient step.\footnote{This idea shares some analogies with the \emph{adapted parameter} in the meta-learning setting~\cite{finn2017model}.}

\begin{restatable}[]{thr}{lowerBoundObjective}
	Let $I \in \interval$ and $\overline{I} =\interval \setminus I$. For a vector $\mathbr{v}$, we denote with $\mathbr{v} \rvert_I$ the vector obtained by setting to zero the components in $I$. Let $\mathbr{\theta}^*(\mathbr{\omega}_0) \in \Theta$ the initial parameter. Let $\alpha \ge 0$, $\mathbr{\theta}_I^* (\mathbr{\omega}) =  \mathbr{\theta}_0 + \alpha \nabla_{\mathbr{\theta}} J_{\mathcal{M}_{\mathbr{\omega}}} (\mathbr{\theta}^*(\mathbr{\omega}_0)) \rvert_I$ and $\mathbr{\theta}^* (\mathbr{\omega}) =  \mathbr{\theta}_0 + \alpha \nabla_{\mathbr{\theta}} J_{\mathcal{M}_{\mathbr{\omega}}} (\mathbr{\theta}^*(\mathbr{\omega}_0))$. Then, under Assumption~\ref{ass:identifiability}, we have:
	\begin{equation*}
		{\ell}({\mathbr{\theta}_I^*}(\mathbr{\omega})) - {\ell}({\mathbr{\theta}^*}(\mathbr{\omega})) \ge \frac{\lambda_{\min} \alpha^2}{2} \left\| \nabla_{\mathbr{\theta}} J_{\mathcal{M}_{\mathbr{\omega}}} (\mathbr{\theta}^*(\mathbr{\omega}_0)) \rvert_{\overline{I}} \right\|_2^2.
	\end{equation*}
\end{restatable}

Thus, we maximize the $L^2$--norm of the gradient components that correspond to the parameters we want to test. Since we have at our disposal only samples $\mathcal{D}$ collected with the current policy $\pi_{\mathbr{\theta}^*(\mathbr{\omega}_0)}$ and in the current environment $\mathbr{\omega}_0$, we have to perform an off--distribution optimization over $\mathbr{\omega}$. To this end, we employ an approach analogous to that of~\citet{metelli2018policy} where we optimize the empirical version of the objective with a penalization that accounts for the distance between the distribution over trajectories:
\begin{equation}\label{eq:obj}
\resizebox{0.88\linewidth}{!}{$
\displaystyle \mathcal{C}_I(\mathbr{\omega}/\mathbr{\omega}_0) = \left\| \widehat{\nabla}_{\mathbr{\theta}} J_{\mathcal{M}_{\mathbr{\omega}/\mathbr{\omega}_0}}(\mathbr{\theta}^*(\mathbr{\omega}_0)) \rvert_{\overline{I}} \right\|_2^2 - \zeta \sqrt{\frac{	\widehat{d}_2 (\mathbr{\omega} \| \mathbr{\omega}_0) }{n}},$}
\end{equation}
where $\widehat{\nabla}_{\mathbr{\theta}} J_{\mathcal{M}_{\mathbr{\omega}/\mathbr{\omega}_0}}(\mathbr{\theta}^*(\mathbr{\omega}_0))$ is an off-distribution estimator of the gradient $\nabla_{\mathbr{\theta}} J_{\mathcal{M}_{\mathbr{\omega}}} (\mathbr{\theta}^*(\mathbr{\omega}_0))$ using samples collected with $\mathbr{\omega}_0$, $\widehat{d}_2$ is the estimated 2-\Renyi divergence~\cite{erven2014renyi} that works as a penalization to discourage too large updates and $\zeta \ge 0$ is a regularization parameter. The expression of the estimated gradient, 2-\Renyi divergence and the pseudocode are reported in Appendix~\ref{apx:configuration}.

\begin{figure*}
\begin{minipage}[t]{0.24\textwidth}
	\centering
	\includegraphics[scale=0.92]{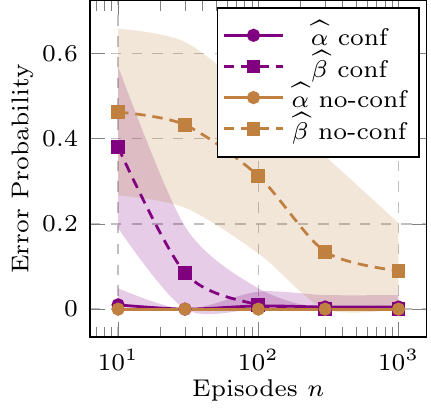}
	\captionof{figure}{\emph{Discrete Grid World}: $\widehat{\alpha}$ and $\widehat{\beta}$ error for \emph{conf} and \emph{no-conf} cases varying the number of episodes. 25 runs 95\% c.i.}\label{fig:gridworld}
\end{minipage}%
\hfill
\begin{minipage}[t]{0.48\textwidth}
\includegraphics[scale=0.92]{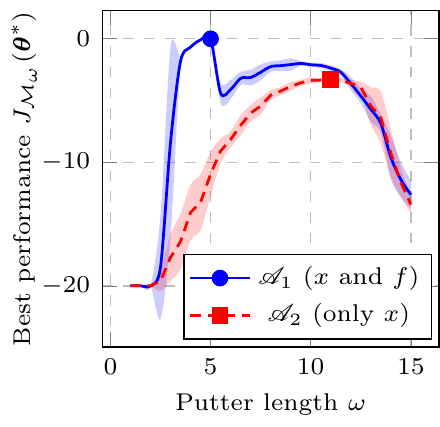}
\includegraphics[scale=0.92]{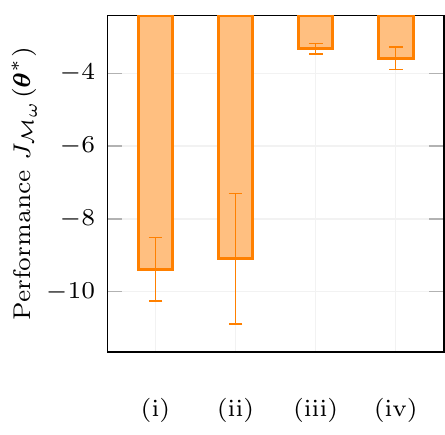}
\captionof{figure}{\emph{Mingolf}: Performance of the optimal policy varying the putter length $\omega$ for agents $\mathscr{A}_1$ and $\mathscr{A}_2$ (left) and performance of the optimal policy for agent $\mathscr{A}_2$ with four different strategies for selecting $\omega$ (right). 100 runs 95\% c.i.}\label{fig:minigolf}
\end{minipage}%
\hfill
\begin{minipage}[t]{0.24\textwidth}
	\centering
	\includegraphics[scale=0.92]{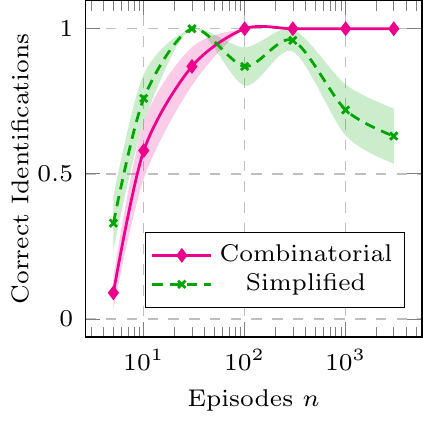}
	\captionof{figure}{\emph{Simulated Car Driving}: fraction of correct identifications varying the number of episodes. 100 runs 95\% c.i.}\label{fig:car}
\end{minipage}%
\end{figure*}
\section{Experimental Evaluation}\label{sec:experiments}
In this section, we present the experimental evaluation of the identification rules in three RL domains.
To set the values of $c(l)$ we resort to the Wilk's asymptotic approximation (Theorem~\ref{thr:wilks}) to enforce (asymptotic) guarantees on the type I error. For Identification Rule~\ref{ir:complete} we perform $2^d$ statistical tests by using the same dataset $\mathcal{D}$. Thus, we partition $\delta$ using Bonferroni correction and setting $c(l) = \chi^2_{l,1-{\delta}/{2^d}}$, where $\chi^2_{l,\xi}$ is the $\xi$--quintile of a chi square distribution with $l$ degrees of freedom. Instead, for Identification Rule~\ref{ir:simplified}, we perform $d$ statistical test, and thus, we set $c(1) = \chi^2_{1,1-{\delta}/{d}}$.

\subsection{Discrete Grid World}
The grid world environment is a simple representation of a two-dimensional world (5$\times$5 cells) in which an agent has to reach a target position by moving in the four directions. The goal of this set of experiments is to show the advantages of configuring the environment when performing the policy space identification using rule~\ref{ir:simplified}. The initial position of the agent and the target position are drawn at the beginning of each episode from a Boltzmann distribution $\mu_{\mathbr{\omega}}$. The agent plays a Boltzmann linear policy $\pi_{\mathbr{\theta}}$ with binary features $\mathbr{\phi}$ indicating its current row and column and the row and column of the goal.\footnote{The features are selected to fulfill Lemma~\ref{lemma:identExp}.} For each run, the agent can control a subset $I^*$ of the parameters $\mathbr{\theta}_{I^*}$ associated with those features, which is randomly selected. Furthermore, the supervisor can configure the environment by changing the parameters $\mathbr{\omega}$ of the initial state distribution $\mu_{\mathbr{\omega}}$. Thus, the supervisor can induce the agent to explore certain regions of the grid world and, consequently, change the relevance of the corresponding parameters in the optimal policy.

Figure~\ref{fig:gridworld} shows the empirical $\widehat{\alpha}$ and $\widehat{\beta}$, \ie the fraction of parameters that the agent does not control that are wrongly selected and the fraction of those the agent controls that are not selected respectively, as a function of the number $n$ of episodes used to perform the identification. We compare two cases: \emph{conf} where the identification is carried out by also configuring the environment, \ie optimizing Equation~\eqref{eq:obj}, and \emph{no-conf} in which the identification is performed in the original environment only. In both cases, we can see that $\widehat{\alpha}$ is almost  independent of the number of samples, as it is directly controlled by the critical value $c(1)$. Differently, $\widehat{\beta}$ decreases as the number of samples increases, \ie the power of the test $1-\widehat{\beta}$ increases with $n$. Remarkably, we observe that configuring the environment gives a significant advantage in understanding the parameters controlled by the agent \wrt using a fixed environment, as $\widehat{\beta}$ decreases faster in the \emph{conf} case. This phenomenon also justifies empirically our choice of objective (Equation~\eqref{eq:obj}) for selecting the new environment. Hyperparameters, further experimental results, together with experiments on a continuous version of the grid world, are reported in Appendix~\ref{apx:discreteGridworld}--\ref{apx:continuousGridworld}.

\subsection{Minigolf}
In the Minigolf environment~\cite{lazaric2008reinforcement}, an agent hits a ball using a putter with the goal of reaching the hole in the minimum number of attempts. Surpassing the hole causes the termination of the episode and a large penalization. The agent selects the force applied to the putter by playing a Gaussian policy linear in some polynomial features (complying to Lemma~\ref{lemma:identExp}) of the distance from the hole ($x$) and the friction of the green ($f$). We consider two agents: $\mathscr{A}_1$ has access to both the $x$ and $f$ whereas $\mathscr{A}_2$ knows only $x$. Thus, we expect that $\mathscr{A}_1$ learns a policy that allows reaching the hole in a smaller number of hits, compared to $\mathscr{A}_2$, as it can calibrate force according to friction; whereas $\mathscr{A}_2$ has to be more conservative, being unaware of $f$. There is also a supervisor in charge of selecting, for the two agents, the best putter length $\omega$, \ie the configurable parameter of the environment. In this experiment, we want to highlight that knowing the policy space might be of crucial importance when learning in a Conf--MDP.

Figure~\ref{fig:minigolf}-left shows the performance of the optimal policy as a function of the putter length $\omega$. We can see that for agent $\mathscr{A}_1$ the optimal putter length is $\omega^*_{\mathscr{A}_1}=5$ while for agent $\mathscr{A}_2$ is $\omega^*_{\mathscr{A}_2}=11.5$.
Figure~\ref{fig:minigolf}-right compares the performance of the optimal policy of agent $\mathscr{A}_2$ when the putter length $\omega$ is chosen by the supervisor using four different strategies. In (i) the configuration is sampled uniformly in the interval $[1,15]$. In (ii) the supervisor employs the optimal configuration for agent $\mathscr{A}_1$ ($\omega=5$), \ie assuming the agent is aware of the friction. (iii) is obtained by selecting the optimal configuration of the policy space produced by using our identification rule~\ref{ir:simplified}. Finally, (iv) is derived by employing an oracle that knows the true agent's policy space ($\omega=11.5$). We can see that the performance of the identification procedure (iii) is comparable with that of the oracle (iv) and notably higher than the performance when employing an incorrect policy space (ii). Hyperparameters and additional experiments are reported in Appendix~\ref{apx:experimentsMinigolf}.

\subsection{Simulated Car Driving}
We consider a simple version of a car driving simulator, in which the agent has to reach the end of a road in the minimum amount of time, avoiding running off-road. The agent perceives its speed, four sensors placed at different angles that provide distance from the edge of the road and it can act on acceleration and steering. The purpose of this experiment is to show a case in which the identifiability assumption (Assumption~\ref{ass:identifiability}) may not be satisfied. The policy $\pi_{\mathbr{\theta}}$ is modeled as a Gaussian policy whose mean is computed via a single hidden layer neural network with 8 neurons. Some of the sensors are not available to the agent, our goal is to identify which ones the agent can perceive. In Figure~\ref{fig:car}, we compare the performance of the Identification Rules~\ref{ir:complete} (Combinatorial) and~\ref{ir:simplified} (Simplified), showing the fraction of runs that correctly identify the policy space. We note that, while for a small number of samples the simplified rule seems to outperform, when the number of samples increases the combinatorial rule displays remarkable stability, approaching the correct identification in all the runs. This is explained by the fact that, when multiple representations for the same policy are possible, considering one parameter at a time might induce the simplified rule to select a wrong set of parameters. Hyperparameters are reported in Appendix~\ref{apx:experimentsCar}.

%

%
%
%
%


\section{Discussion and Conclusions}
In this paper, we addressed the problem of identifying the policy space of an agent by simply observing its behavior when playing the optimal policy. We introduced two identification rules, both based on the GLR test, which can be applied to select the parameters controlled by the agent. Additionally, we have shown how to use the configurability property of the environment to enhance the effectiveness of identification rules. The experimental evaluation highlights two essential points. First, the identification of the policy space brings advantages to the learning process in a Conf--MDP, helping to choose wisely the most suitable environment configuration. Second, we have illustrated that configuring the environment is beneficial to speed up the identification process. We believe that this work opens numerous future research directions, both theoretical, such as the analysis of the combinatorial identification rule, and empirical, like the application of our identification rules to  imitation learning settings.

\small
\bibliographystyle{aaai}
\bibliography{biblio}

\onecolumn
\normalsize
\appendix
\setlength{\abovedisplayskip}{4pt}
\setlength{\belowdisplayskip}{4pt}
%
%
\section{Proofs and Derivations}\label{apx:proofs}
In this appendix, we report the proofs and derivations of the results presented in the main paper.

\subsection{Gaussian and Boltzmann Linear Policies as Exponential Family distributions}\label{apx:expFamily}
In this appendix, we show how a multivariate Gaussian with fixed covariance and a Boltzmann policy, both linear in the state features $\mathbr{\phi}(s)$ can be cast into Definition~\ref{defi:expFamily}. We are going to make use of the following identities regarding the Kronecker product~\cite{petersen2008matrix}:
\begin{align}
	& \mathrm{vec}(\mathbr{AXB}) = \left( \mathbr{B}^T \otimes \mathbr{A} \right) \mathrm{vec} (\mathbr{X}) \label{eq:kron}\\
	& \mathbr{a}^T \mathbr{XBX}^T \mathbr{c} = \mathrm{vec}(\mathbr{X})^T \left( \mathbr{B} \otimes \mathbr{ca}^T \right) \mathrm{vec}(\mathbr{X}), \label{eq:kron2}
\end{align}
where $\mathrm{vec} (\mathbr{X})$ is the \emph{vectorization} of matrix $\mathbr{X}$ obtained by stacking the columns of $\mathbr{X}$ into a single column vector.

\subsubsection{Multivariate Linear Gaussian Policy with fixed covariance}
The typical representation of a multivariate linear Gaussian policy is given by the following probability density function:
\begin{equation*}
	\pi_{\widetilde{\mathbr{\theta}}}(\mathbr{a}|s) = \frac{1}{(2\pi)^{\frac{k}{2}} \det(\mathbr{\Sigma})^{\frac{1}{2}}} \exp\left\{ -\frac{1}{2} (\mathbr{a} - \widetilde{\mathbr{\theta}}\mathbr{\phi}(s))^T\mathbr{\Sigma}^{-1}(\mathbr{a} - \widetilde{\mathbr{\theta}}\mathbr{\phi}(s)) \right\},
\end{equation*}
where $\widetilde{\mathbr{\theta}} \in \mathbb{R}^{k \times q}$ is a properly sized matrix. Recalling Definition~\ref{defi:expFamily}, we rephrase the previous equation as:
\begin{equation*}
	\pi_{\widetilde{\mathbr{\theta}}}(\mathbr{a}|s) = \frac{1}{(2\pi)^{\frac{k}{2}} \det(\mathbr{\Sigma})^{\frac{1}{2}}} \exp\left\{ -\frac{1}{2} \mathbr{a}^T\mathbr{\Sigma}^{-1}\mathbr{a}\right\} \exp\left\{  \mathbr{\phi}(s)^T \widetilde{\mathbr{\theta}}^T \mathbr{\Sigma}^{-1}  \mathbr{a} - \frac{1}{2} \mathbr{\phi}(s)^T \widetilde{\mathbr{\theta}}^T \mathbr{\Sigma}^{-1} \widetilde{\mathbr{\theta}} \mathbr{\phi}(s) \right\}.
\end{equation*}
Recalling the identities at Equation~\eqref{eq:kron} and~\eqref{eq:kron2}  and observing that $  \mathbr{\phi}(s)^T \widetilde{\mathbr{\theta}}^T \mathbr{\Sigma}^{-1}  \mathbr{a}$ and $\mathbr{\phi}(s)^T \widetilde{\mathbr{\theta}}^T \mathbr{\Sigma}^{-1} \widetilde{\mathbr{\theta}} \mathbr{\phi}(s)$ are  scalar, we can rewrite:
\begin{align*}
	&  \mathbr{\phi}(s)^T \widetilde{\mathbr{\theta}}^T \mathbr{\Sigma}^{-1}  \mathbr{a} = \mathrm{vec} \left(  \mathbr{\phi}(s)^T \widetilde{\mathbr{\theta}}^T \mathbr{\Sigma}^{-1}  \mathbr{a} \right) = \left( \mathbr{a}^T \mathbr{\Sigma}^{-1} \otimes \mathbr{\phi}(s)^T \right) \mathrm{vec}\left( \widetilde{\mathbr{\theta}}^T \right) = \mathrm{vec}\left(\widetilde{\mathbr{\theta}}^T\right)^T \left( \mathbr{\Sigma}^{-1}\mathbr{a} \otimes \mathbr{\phi}(s)  \right)  \\
	& \mathbr{\phi}(s)^T \widetilde{\mathbr{\theta}}^T \mathbr{\Sigma}^{-1} \widetilde{\mathbr{\theta}} \mathbr{\phi}(s) = \mathrm{vec}\left(\widetilde{\mathbr{\theta}}^T\right)^T \left(\mathbr{\Sigma}^{-1} \otimes  \mathbr{\phi}(s) \mathbr{\phi}(s)^T  \right) \mathrm{vec}\left(\widetilde{\mathbr{\theta}}^T \right).
\end{align*}
Now, by redefining the parameter of the exponential family distribution $\mathbr{\theta} = \mathrm{vec}\left(\widetilde{\mathbr{\theta}}^T\right)$ we state the following definitions to comply with Definition~\ref{defi:expFamily}:
\begin{align*}
	& \mathbr{t}(s,\mathbr{a}) = \mathbr{\Sigma}^{-1} \mathbr{a} \otimes \mathbr{\phi}(s)\\
	& h(\mathbr{a}) = \frac{1}{(2\pi)^{\frac{k}{2}} \det(\mathbr{\Sigma})^{\frac{1}{2}}} \exp\left\{ -\frac{1}{2} \mathbr{a}^T\mathbr{\Sigma}^{-1}\mathbr{a}\right\}\\
	& A(\mathbr{\theta},s) = \mathbr{\theta}^T \left(\mathbr{\Sigma}^{-1} \otimes  \mathbr{\phi}(s) \mathbr{\phi}(s)^T  \right) \mathbr{\theta}.
\end{align*}

\subsubsection{Boltzmann Linear Policy}
The Boltzmann policy on a finite set of actions $\{a_1,...,a_{k+1}\}$ is typically represented by means of a matrix of parameters $\widetilde{\mathbr{\theta}} \in \mathbb{R}^{k \times q}$:\footnote{Notice that we are considering a set made of $k+1$ actions but the matrix $\widetilde{\mathbr{\theta}}$ has only $k$ rows. This allows enforcing the identifiability property, otherwise if we had a row for each of the $k+1$ actions we would have multiple representation for the same policy (rescaling the rows by the same amount).} 
\begin{equation*}
	 \pi_{\widetilde{\mathbr{\theta}}}(a_i|s) = \begin{cases} \frac{\exp \left\{\widetilde{\mathbr{\theta}}_i^T \mathbr{\phi}(s)\right\}}{1+\sum_{j=1}^{k} \exp\left\{\widetilde{\mathbr{\theta}}_j^T \mathbr{\phi}(s) \right\}} & \text{if } i \le k \\
	 \frac{1}{1+\sum_{j=1}^{k} \exp\left\{\widetilde{\mathbr{\theta}}_j^T \mathbr{\phi}(s) \right\}} & \text{if } i = k+1
	 \end{cases},
\end{equation*}
where with $\widetilde{\mathbr{\theta}}_i$ we denote the $i$-th row of matrix $\widetilde{\mathbr{\theta}}$. In order to comply to Definition~\ref{defi:expFamily}, we rewrite the density function in the following form:
\begin{equation*}
	\pi_{\widetilde{\mathbr{\theta}}}(a_i|s) = \begin{cases} \exp \left\{\widetilde{\mathbr{\theta}}_i^T \mathbr{\phi}(s) - \log \left( \exp\{0\} +\sum_{j=1}^{k} \exp\left\{\widetilde{\mathbr{\theta}}_j^T \mathbr{\phi}(s) \right\} \right) \right\} & \text{if } i \le k \\
	 \exp \left\{0 - \log \left( \exp\{0\} +\sum_{j=1}^{k} \exp\left\{\widetilde{\mathbr{\theta}}_j^T \mathbr{\phi}(s) \right\} \right) \right\} & \text{if } i = k+1
	 \end{cases}.
\end{equation*}
By introducing the vector $\mathbr{e}_i$ as the $i$--th vector of the canonical basis of $\mathbb{R}^k$, \ie the vector having $1$ in the $i$--th component and $0$ elsewhere, and recalling the definition of Kronecker product, we can derive the following identity for $i \le k$:
\begin{equation*}
\widetilde{\mathbr{\theta}}_i^T \mathbr{\phi}(s) = \mathrm{vec}\left(\widetilde{\mathbr{\theta}}^T\right)^T \left( \mathbr{e}_i \otimes \mathbr{\phi}(s) \right).
\end{equation*}
In the case $i=k$ it is sufficient to replace the previous term with the zero vector $\mathbr{0}$.
Therefore, by renaming $\mathbr{\theta} = \mathrm{vec}\left(\widetilde{\mathbr{\theta}}^T\right)$ we can make the following assignments in order to get the relevant quantities in Definition~\ref{defi:expFamily}:
\begin{align*}
	& \mathbr{t}(s,a_i) = \begin{cases} \mathbr{e}_i \otimes \mathbr{\phi}(s) & \text{if } i \le k \\
				\mathbr{0} & \text{if } i = k+1 \end{cases} \\
	& h(a_i) = 1\\
	& A(\mathbr{\theta},s) = \log \left( 1 +\sum_{j=1}^{k} \exp\left\{\mathbr{\theta}^T \left( \mathbr{e}_j \otimes \mathbr{\phi}(s) \right) \right\} \right).
\end{align*}

\subsection{Results on Exponential Family}\label{apx:expFamilyResults}
In this appendix, we derive several results that are used in Section~\ref{sec:analysis}, concerning policies belonging to the exponential family, as in Definition~\ref{defi:expFamily}.

\subsubsection{Fisher Information Matrix}\label{apx:fisher}
We start by providing an expression of the Fisher Information matrix (FIM) for the specific case of the exponential family, that we are going to use extensively in the derivation. We first define the FIM for a fixed state and then we provide its expectation under the state distribution $d_{\mu}^{\pi^*}$. For any state $s \in \mathcal{S}$, we define the FIM induced by $\pi_{\mathbr{\theta}}(\cdot|s)$ as:
\begin{equation}\label{eq:fisher}
	\mathcal{F}(\mathbr{\theta},s) =  \E_{a \sim \pi_{\mathbr{\theta}} (\cdot|s)} \left[ \nabla_{\mathbr{\theta}} \log \pi_{\mathbr{\theta}}(a|s) \nabla_{\mathbr{\theta}} \log \pi_{\mathbr{\theta}}(a|s)^T \right].
\end{equation}
We can derive the following immediate result.
\begin{lemma}\label{lemma:fisherDeriv}
	For a policy $\pi_{\mathbr{\theta}}$ belonging to the exponential family, as in  Definition~\ref{defi:expFamily}, the FIM for state $s \in \mathcal{S}$ is given by the covariance matrix of the sufficient statistic:
	\begin{equation*}
	 	\mathcal{F}(\mathbr{\theta},s) = \E_{a \sim \pi_{\mathbr{\theta}} (\cdot|s)} \left[ \mathbr{\overline{t}}(s,a,\mathbr{\theta}) \mathbr{\overline{t}}(s,a,\mathbr{\theta})^T\right] = \Cov_{a \sim \pi_{\mathbr{\theta}} (\cdot|s)} \left[ \mathbr{t}(s,a) \right].
	\end{equation*}
\end{lemma}
\begin{proof}
	Let us first compute the gradient log-policy for the exponential family:
	\begin{align}
		\nabla_{\mathbr{\theta}} \log \pi_{\mathbr{\theta}}(a|s) &= \mathbr{t}(s,a) - \nabla_{\mathbr{\theta}}  A(\mathbr{\theta},s) \notag\\
		& =  \mathbr{t}(s,a) - \frac{\int_{\mathcal{A}} \mathbr{t}(s,\overline{a}) h(\overline{a})  \exp \left\{ \mathbr{\theta}^T \mathbr{t}(s,\overline{a}) \right\}\mathrm{d} \overline{a}}{\int_{\mathcal{A}} h(\overline{a})  \exp \left\{ \mathbr{\theta}^T \mathbr{t}(s,\overline{a}) \right\}\mathrm{d} \overline{a}} \label{eq:grad-log-pi} \\
		& = \mathbr{t}(s,a) - \E_{\overline{a} \sim \pi_{\mathbr{\theta}} (\cdot|s)} \left[\mathbr{t}(s,\overline{a}) \right] = \mathbr{\overline{t}}(s,a,\mathbr{\theta}). \notag
	\end{align}
	Now, we just need to apply the definition given in Equation~\eqref{eq:fisher} and to recall the definition of covariance matrix:
	\begin{align*}
		\mathcal{F}(\mathbr{\theta},s) & = \E_{a \sim \pi_{\mathbr{\theta}} (\cdot|s)} \left[ \mathbr{\overline{t}}(s,a,\mathbr{\theta}) \mathbr{\overline{t}}(s,a,\mathbr{\theta})^T\right] \\
		& = \E_{a \sim \pi_{\mathbr{\theta}} (\cdot|s)} \left[ \left(\mathbr{t}(s,a) - \E_{\overline{a} \sim \pi_{\mathbr{\theta}} (\cdot|s)} \left[\mathbr{t}(s,\overline{a}) \right] \right)  \left( \mathbr{t}(s,a) - \E_{\overline{a} \sim \pi_{\mathbr{\theta}} (\cdot|s)} \left[\mathbr{t}(s,\overline{a}) \right] \right)^T \right]  = \Cov_{a \sim \pi_{\mathbr{\theta}} (\cdot|s)} \left[ \mathbr{t}(s,a) \right].
	\end{align*}
\end{proof}

We now define the expected FIM $\mathcal{F}(\mathbr{\theta})$ and its corresponding estimator $\widehat{\mathcal{F}}(\mathbr{\theta})$ under the $\gamma$--discounted stationary distribution induced by the agent's policy $\pi^*$:
\begin{equation}
	{\mathcal{F}}(\mathbr{\theta}) = \E_{s \sim d_{\mu}^{\pi^*}} \left[ \E_{a \sim \pi_{\mathbr{\theta}}(\cdot|s)} \left[\mathbr{\overline{t}}(s,a)\mathbr{\overline{t}}(s,a)^T\right] \right], \qquad \widehat{\mathcal{F}}(\mathbr{\theta}) = \frac{1}{n} \sum_{i=1}^n \E_{a \sim \pi_{\mathbr{\theta}}(\cdot|s)} \left[\mathbr{\overline{t}}(s_i,a)\mathbr{\overline{t}}(s_i,a)^T\right].
\end{equation}

Finally, we provide a sufficient condition to ensure that the FIM ${\mathcal{F}}(\mathbr{\theta})$ is non singular in the case of Gaussian and Boltzmann linear policies.
%
%
\begin{prop}\label{prop:fisherForGaussBoltz}
	If the second moment matrix of the feature vector $\E_{s \sim d_{\mu}^{\pi^*}} \left[\mathbr{\phi}(s)\mathbr{\phi}(s)^T \right]$ is non--singular, the identifiability condition of Lemma~\ref{lemma:identExp} is fulfilled by the Gaussian and Boltzmann linear policies for all $\mathbr{\theta} \in \Theta$, provided that each action is played with non--zero probability for the Boltzmann policy.
\end{prop}

\begin{proof}
	Let us start with the Boltzmann policy and consider the expression of $\mathbr{\overline{t}}(s,a_i)$ with $i\in \{1,...,k\}$:
	\begin{align*}
		\mathbr{\overline{t}}(s,a_i, \mathbr{\theta}) & = \mathbr{t}(s,a_i) - \E_{\overline{a} \sim \pi_{\mathbr{\theta}}(\cdot|s)} \left[\mathbr{t}(s,\overline{a}) \right] \\
				& = \mathbr{e}_i \otimes \mathbr{\phi}(s) - \sum_{j=1}^k \pi_{\mathbr{\theta}}(a_i|s) \mathbr{e}_i \otimes \mathbr{\phi}(s) \\
				& = \left(\mathbr{e}_i - \mathbr{\pi} \right) \otimes \mathbr{\phi}(s),
	\end{align*}
	where $\mathbr{\pi}$ is a vector defined as $\mathbr{\pi} = \left(\pi_{\mathbr{\theta}}(a_1|s) ,...,\pi_{\mathbr{\theta}}(a_k|s)  \right)^T$ and we exploited the distributivity of the Kronecker product. While for $i=k+1$, we have $\left(\mathbr{0} - \mathbr{\pi} \right) \otimes \mathbr{\phi}(s)$. For the sake of the proof, let us define $\widetilde{\mathbr{e}}_i = \mathbr{e}_i$ if $i \le k$ and $\widetilde{\mathbr{e}}_{k+1} = \mathbr{0}$. Let us compute the FIM:
	\begin{align*}
		\mathcal{F}(\mathbr{\theta}) &= \E_{a \sim \pi_{\mathbr{\theta}} (\cdot|s)} \left[ \mathbr{\overline{t}}(s,a,\mathbr{\theta}) \mathbr{\overline{t}}(s,a,\mathbr{\theta})^T\right] \\
		& = \E_{a \sim \pi_{\mathbr{\theta}} (\cdot|s)} \left[ \left(\left(\widetilde{\mathbr{e}_i} - \mathbr{\pi} \right) \otimes \mathbr{\phi}(s) \right) \left(\left(\widetilde{\mathbr{e}_i} - \mathbr{\pi} \right) \otimes \mathbr{\phi}(s) \right)^T \right] \\
		& = \E_{a \sim \pi_{\mathbr{\theta}} (\cdot|s)} \left[ \left(\widetilde{\mathbr{e}_i} - \mathbr{\pi} \right)\left(\widetilde{\mathbr{e}_i} - \mathbr{\pi} \right)^T \otimes \mathbr{\phi}(s)\mathbr{\phi}(s)^T \right] \\
		& = \E_{a \sim \pi_{\mathbr{\theta}} (\cdot|s)} \left[ \left(\widetilde{\mathbr{e}_i} - \mathbr{\pi} \right)\left(\widetilde{\mathbr{e}_i} - \mathbr{\pi} \right)^T  \right] \otimes \mathbr{\phi}(s)\mathbr{\phi}(s)^T \\
		& = \left( \E_{a \sim \pi_{\mathbr{\theta}} (\cdot|s)} \left[ \widetilde{\mathbr{e}_i}\widetilde{\mathbr{e}_i}^T \right] - \mathbr{\pi}\mathbr{\pi}^T \right) \otimes \mathbr{\phi}(s)\mathbr{\phi}(s)^T \\
		& = \left( \mathrm{diag} (\mathbr{\pi}) - \mathbr{\pi}\mathbr{\pi}^T \right) \otimes \mathbr{\phi}(s)\mathbr{\phi}(s)^T,
	\end{align*}
	where we exploited the distributivity of the Kroneker product, observed that $\E_{a \sim \pi_{\mathbr{\theta}} (\cdot|s)} \left[ \widetilde{\mathbr{e}_i}\right] = \mathbr{\pi}$ and $\E_{a \sim \pi_{\mathbr{\theta}} (\cdot|s)} \left[ \widetilde{\mathbr{e}_i}\widetilde{\mathbr{e}_i}^T \right] = \mathrm{diag} (\mathbr{\pi})$. Let us now consider the matrix:
	\begin{equation*}
		 \mathrm{diag} (\mathbr{\pi}) - \mathbr{\pi}\mathbr{\pi}^T = \begin{pmatrix}
		 	\pi_{\mathbr{\theta}}(a_1|s) - \pi_{\mathbr{\theta}}(a_1|s)^2 & -\pi_{\mathbr{\theta}}(a_1|s)\pi_{\mathbr{\theta}}(a_2|s) & \dots & -\pi_{\mathbr{\theta}}(a_1|s)\pi_{\mathbr{\theta}}(a_k|s) \\
		 	-\pi_{\mathbr{\theta}}(a_1|s)\pi_{\mathbr{\theta}}(a_2|s) & \pi_{\mathbr{\theta}}(a_2|s) - \pi_{\mathbr{\theta}}(a_2|s)^2 & \dots & -\pi_{\mathbr{\theta}}(a_2|s)\pi_{\mathbr{\theta}}(a_k|s) \\
		 	\vdots & \vdots & \ddots & \vdots \\
		 	-\pi_{\mathbr{\theta}}(a_1|s)\pi_{\mathbr{\theta}}(a_k|s) & -\pi_{\mathbr{\theta}}(a_2|s)\pi_{\mathbr{\theta}}(a_k|s) & \dots & \pi_{\mathbr{\theta}}(a_k|s) - \pi_{\mathbr{\theta}}(a_k|s)^2  \\
		 \end{pmatrix}.
	\end{equation*}
	Consider a generic row $i \in \{1,...,k\}$. The element on the diagonal is $\pi_{\mathbr{\theta}}(a_i|s) - \pi_{\mathbr{\theta}}(a_i|s)^2 = \pi_{\mathbr{\theta}}(a_i|s) \left(1- \pi_{\mathbr{\theta}}(a_i|s)\right)$, while the absolute sum of the elements out of the diagonal is:
	\begin{equation*}
		\pi_{\mathbr{\theta}}(a_i|s) \sum_{j \in \{1,...k\} \wedge j \neq i} \pi_{\mathbr{\theta}}(a_j|s) = \pi_{\mathbr{\theta}}(a_i|s) \left( 1 - \pi_{\mathbr{\theta}}(a_i|s) - \pi_{\mathbr{\theta}}(a_{k+1} |s) \right).
	\end{equation*}
	Therefore, if all actions are played with non--zero probability, \ie $\pi_{\mathbr{\theta}}(a_i|s) > 0$ for all $i \in \{1,...,k+1\}$ it follows that the matrix is strictly diagonally dominant by rows and thus it is positive definite. If also $\E_{s \sim d_{\mu}^{\pi^*}} \left[\mathbr{\phi}(s)\mathbr{\phi}(s)^T \right]$ is positive definite, for the properties of the Kroneker product, the FIM is positive definite.
	
	Let us now focus on the Gaussian policy. Let $\mathbr{a} \in \mathbb{R}^d$ and denote $\mathbr{\mu}(s) = \E_{\mathbr{a} \sim \pi_{\mathbr{\theta}}(\cdot|s)} \left[ \mathbr{a} \right]$:
	\begin{align*}
		\mathbr{\overline{t}}(s,\mathbr{a}, \mathbr{\theta}) & = \mathbr{t}(s,\mathbr{a}) - \E_{\overline{\mathbr{a}} \sim \pi_{\mathbr{\theta}}(\cdot|s)} \left[\mathbr{t}(s,\overline{\mathbr{a}}) \right] = \mathbr{\Sigma}^{-1}\left(\mathbr{a} - \mathbr{\mu}(s)\right) \otimes \mathbr{\phi}(s).
	\end{align*}
	Let us compute the FIM:
	\begin{align*}
		\mathcal{F}(\mathbr{\theta}) &= \E_{a \sim \pi_{\mathbr{\theta}} (\cdot|s)} \left[ \mathbr{\overline{t}}(s,a,\mathbr{\theta}) \mathbr{\overline{t}}(s,a,\mathbr{\theta})^T\right] \\
		& = \E_{a \sim \pi_{\mathbr{\theta}} (\cdot|s)} \left[ \left( \mathbr{\Sigma}^{-1}\left(\mathbr{a} - \mathbr{\mu}(s)\right) \otimes \mathbr{\phi}(s) \right) \left( \mathbr{\Sigma}^{-1}\left(\mathbr{a} - \mathbr{\mu}(s)\right) \otimes \mathbr{\phi}(s) \right)^T\right] \\
		& = \E_{a \sim \pi_{\mathbr{\theta}} (\cdot|s)} \left[ \mathbr{\Sigma}^{-1}\left(\mathbr{a} - \mathbr{\mu}(s)\right)\left(\mathbr{a} - \mathbr{\mu}(s)\right)^T\mathbr{\Sigma}^{-1} \otimes \mathbr{\phi}(s)\mathbr{\phi}(s)^T\right] \\
		& = \mathbr{\Sigma}^{-1} \E_{a \sim \pi_{\mathbr{\theta}} (\cdot|s)} \left[ \left(\mathbr{a} - \mathbr{\mu}(s)\right)\left(\mathbr{a} - \mathbr{\mu}(s)\right)^T \right]\mathbr{\Sigma}^{-1} \otimes \mathbr{\phi}(s)\mathbr{\phi}(s)^T \\
		& = \mathbr{\Sigma}^{-1} \mathbr{\Sigma} \mathbr{\Sigma}^{-1} \otimes \mathbr{\phi}(s)\mathbr{\phi}(s)^T  = \mathbr{\Sigma}^{-1}\otimes \mathbr{\phi}(s)\mathbr{\phi}(s)^T.
	\end{align*} 
	If $\mathbr{\Sigma}$ has finite values, then $\mathbr{\Sigma}^{-1}$ will be positive definite and additionally, considering that $\E_{s \sim d_{\mu}^{\pi^*}} \left[\mathbr{\phi}(s)\mathbr{\phi}(s)^T \right]$ is positive definite, we have that the FIM is positive definite.
\end{proof}

\subsubsection{Subgaussianity Assumption}\label{apx:subgaussian}

From Assumption~\ref{ass:subGauss}, we can prove the following result that upper bounds the maximum eigenvalue $\lambda_{\max}$ of the Fisher information matrix with the subgaussianity parameter $\sigma$.

\begin{lemma}\label{lemma:fisherMaxEig}
	Under Assumption~\ref{ass:subGauss}, for any $\mathbr{\theta} \in \Theta$ and for any $s \in \mathcal{S}$ the maximum eigenvalue of the Fisher Information matrix $\mathcal{F}(\mathbr{\theta},s)$ is upper bounded by $d\sigma^2$.
\end{lemma}

\begin{proof}
	Recall that the maximum eigenvalue of a matrix $\mathbr{A}$ can be computed as $\sup_{\mathbr{x} : \left\| \mathbr{x} \right\|_2 \le 1} \mathbr{x}^T\mathbr{Ax}$ and the norm of a vector $\mathbr{y}$ can be computed as $\sup_{\mathbr{x} : \left\| \mathbr{x} \right\|_2 \le 1} \mathbr{x}^T \mathbr{y}$. Consider now the derivation for a generic $\mathbr{x} \in \mathbb{R}^d$ such that $\left\| \mathbr{x} \right\|_2 \le 1$:
	\begin{align*}
		\mathbr{x}^T \mathcal{F}(\mathbr{\theta},s) \mathbr{x} & = \mathbr{x}^T \E_{a \sim \pi_{\mathbr{\theta}} (\cdot|s)} \left[  \mathbr{\overline{t}}(s,a,\mathbr{\theta}) \mathbr{\overline{t}}(s,a,\mathbr{\theta})^T \right] \mathbr{x}\\ 
		& =  \E_{a \sim \pi_{\mathbr{\theta}} (\cdot|s)} \left[  \mathbr{x}^T \mathbr{\overline{t}}(s,a,\mathbr{\theta}) \mathbr{\overline{t}}(s,a,\mathbr{\theta})^T \mathbr{x} \right] \\
		& = \E_{a \sim \pi_{\mathbr{\theta}} (\cdot|s)} \left[ \left( \mathbr{x}^T \mathbr{\overline{t}}(s,a,\mathbr{\theta}) \right)^2 \right] \\
		& \le \E_{a \sim \pi_{\mathbr{\theta}} (\cdot|s)} \left[\left( \sup_{\mathbr{x} : \left\| \mathbr{x} \right\|_2 \le 1} \mathbr{x}^T \mathbr{\overline{t}}(s,a,\mathbr{\theta}) \right)^2 \right] = \E_{a \sim \pi_{\mathbr{\theta}} (\cdot|s)} \left[ \left\| \mathbr{\overline{t}}(s,a,\mathbr{\theta}) \right\|_2^2 \right],
	\end{align*}
	where we employed Lemma~\ref{lemma:fisherDeriv} and upper bounded the right hand side. By taking the supremum over $\mathbr{x} \in \mathbb{R}^d$ such that $\left\| \mathbr{x} \right\|_2 \le 1$ we get:
	\begin{equation}
		\lambda_{\max} \left( \mathcal{F}(\mathbr{\theta},s) \right) = \sup_{\mathbr{x} : \left\| \mathbr{x} \right\|_2 \le 1} \mathbr{x}^T \mathcal{F}(\mathbr{\theta},s) \mathbr{x}  \le \E_{a \sim \pi_{\mathbr{\theta}} (\cdot|s)} \left[ \left\| \mathbr{\overline{t}}(s,a,\mathbr{\theta}) \right\|_2^2 \right].
	\end{equation}
	By applying the first inequality in Remark 2.2 of~\citet{hsu2011tail} and setting $\mathbr{A} = \mathbr{I}$ we get that $\E_{a \sim \pi_{\mathbr{\theta}} (\cdot|s)} \left[ \left\| \mathbr{\overline{t}}(s,a,\mathbr{\theta}) \right\|_2^2 \right] \le d\sigma^2$.
\end{proof}

We now show that the subgaussianity assumption is satisfied by the Boltzmann and Gaussian policies, as defined in Table~\ref{tab:expFamily}, under mild assumptions.

\begin{prop}\label{prop:subGaussExample}
	If the features $\mathbr{\phi}$ are uniformly bounded in norm over the state space, \ie $\Phi_{\max} = \sup_{s \in \mathcal{S}} \left\| \mathbr{\phi}(s) \right\|_2$, then Assumption~\ref{ass:subGauss} is fulfilled by the Boltzmann linear policy with parameter $\sigma = 2 \Phi_{\max}$ and Gaussian linear policy with parameter $\sigma = \frac{\Phi_{\max}}{\sqrt{ \lambda_{\min} \left( \mathbr{\Sigma} \right)}}$.
\end{prop}

\begin{proof}
	Let us start with the Boltzmann policy. From the definition of subgaussianity given in Assumption~\ref{ass:subGauss}, requiring that the random vector $\mathbr{\overline{t}}(s,a_i, \mathbr{\theta})$ is subgaussian with parameter $\sigma$ is equivalent to require that the random (scalar) variable $\frac{1}{\|\mathbr{\alpha}\|_2} \mathbr{\alpha}^T\mathbr{\overline{t}}(s,a_i, \mathbr{\theta})$ is subgaussian with parameter $\sigma$ for any $\mathbr{\alpha} \in \mathbb{R}^d$. Thus, we now bound the term:
	\begin{align*}
		\left| \mathbr{\alpha}^T \mathbr{\overline{t}}(s,a,\mathbr{\theta}) \right| &= \left| \mathbr{\alpha}^T \left(\left(\widetilde{\mathbr{e}}_i - \mathbr{\pi} \right) \otimes \mathbr{\phi}(s) \right) \right| \\
		& = \left\| \mathbr{\alpha} \right\|_2 \left\| \left(\widetilde{\mathbr{e}}_i - \mathbr{\pi} \right) \otimes \mathbr{\phi}(s) \right\|_2 \\
		& = \left\| \mathbr{\alpha} \right\|_2 \left\| \widetilde{\mathbr{e}}_i - \mathbr{\pi} \right\|_2 \left\| \mathbr{\phi}(s) \right\|_2 \\
		& \le  2 \left\| \mathbr{\alpha} \right\|_2 \Phi_{\max} , \\
	\end{align*}
	where we used Cauchy--Swartz inequality, the identity $\left\|\mathbr{x} \otimes \mathbr{y} \right\|_2^2 = \left(\mathbr{x} \otimes \mathbr{y}\right)^T \left(\mathbr{x} \otimes \mathbr{y}\right) = \left(\mathbr{x}^T \mathbr{x} \right) \otimes \left(\mathbr{y}^T \mathbr{y} \right) = \left\| \mathbr{x} \right\|_2^2 \left\| \mathbr{y} \right\|_2^2$ and the inequality $\left\| \widetilde{\mathbr{e}}_i - \mathbr{\pi} \right\|_2^2  \le 2$. Thus, the random variable $\frac{1}{\|\mathbr{\alpha}\|_2} \mathbr{\alpha}^T\mathbr{\overline{t}}(s,a_i, \mathbr{\theta}) \le 2 \Phi_{\max}$ is bounded. Thanks to Hoeffding's lemma we have that the subgaussianity parameter is $\sigma = 2 \Phi_{\max}$.
	
	Let us now consider the Gaussian policy. Let $\mathbr{a} \in \mathbb{R}^d$ and denote with $\mathbr{\mu}(s) = \E_{\mathbr{a} \sim \pi_{\mathbr{\theta}}(\cdot|s)} \left[ \mathbr{a} \right]$:
	\begin{align*}
		\mathbr{\overline{t}}(s,\mathbr{a}, \mathbr{\theta}) & = \mathbr{t}(s,\mathbr{a}) - \E_{\overline{\mathbr{a}} \sim \pi_{\mathbr{\theta}}(\cdot|s)} \left[\mathbr{t}(s,\overline{\mathbr{a}}) \right] = \mathbr{\Sigma}^{-1}\left(\mathbr{a} - \mathbr{\mu}(s)\right) \otimes \mathbr{\phi}(s).
	\end{align*}
	Let us first observe that we can rewrite:
	\begin{align*}
	\mathbr{\alpha}^T \left( \mathbr{\Sigma}^{-1}\left(\mathbr{a} - \mathbr{\mu}(s)\right) \otimes \mathbr{\phi}(s) \right) &= \sum_{i=1}^k \sum_{j=1}^q \alpha_{ij} \left( \mathbr{\Sigma}^{-1}\left(\mathbr{a} - \mathbr{\mu}(s)\right)  \right)_i \phi(s)_j \\
	& = \sum_{i=1}^k \sum_{j=1}^q \alpha_{ij}\phi(s)_j \left( \mathbr{\Sigma}^{-1}\left(\mathbr{a} - \mathbr{\mu}(s)\right)  \right)_i \\
	& = \mathbr{\beta}^T \mathbr{\Sigma}^{-1}\left(\mathbr{a} - \mathbr{\mu}(s)\right),
	\end{align*}
	where $\beta_i = \sum_{j} \alpha_{ij}\phi(s)_j$ for $i \in \{1,...,k\}$. We now proceed with explicit computations:
	\begin{align*}
		\E_{\mathbr{a} \sim \pi_{\mathbr{\theta}}(\cdot|s)} \left[ \exp \left\{ \mathbr{\alpha}^T \mathbr{\overline{t}}(s,\mathbr{a}, \mathbr{\theta}) \right\} \right] & = \E_{\mathbr{a} \sim \pi_{\mathbr{\theta}}(\cdot|s)} \left[ \exp \left\{ \mathbr{\alpha}^T \left( \mathbr{\Sigma}^{-1}\left(\mathbr{a} - \mathbr{\mu}(s)\right) \otimes \mathbr{\phi}(s) \right) \right\} \right] \\
		& =  \E_{\mathbr{a} \sim \pi_{\mathbr{\theta}}(\cdot|s)} \left[ \exp \left\{ \mathbr{\beta}^T \mathbr{\Sigma}^{-1} \left(\mathbr{a} - \mathbr{\mu}(s) \right)  \right\} \right] \\
		& = \int_{\mathbb{R}^d} \frac{\exp\left\{ -\frac{1}{2} (\mathbr{a} - \mathbr{\mu}(s))^T \mathbr{\Sigma}^{-1} (\mathbr{a} - \mathbr{\mu}(s)) \right\}}{(2\pi)^{\frac{k}{2}} \det(\mathbr{\Sigma})^{\frac{1}{2}}} \exp \left\{ \mathbr{\beta}^T  \mathbr{\Sigma}^{-1}\left(\mathbr{a} - \mathbr{\mu}(s)\right) \right\} \mathrm{d} \mathbr{a}.
	\end{align*}
	Now we complete the square:
	\begin{align*}
		-\frac{1}{2} (\mathbr{a} - \mathbr{\mu}(s))^T \mathbr{\Sigma}^{-1} (\mathbr{a} - \mathbr{\mu}(s)) + \mathbr{\beta}^T  \mathbr{\Sigma}^{-1} (\mathbr{a} - \mathbr{\mu}(s)) = - \frac{1}{2} (\mathbr{a} - \mathbr{\mu}(s)-\mathbr{\beta})^T \mathbr{\Sigma}^{-1} (\mathbr{a} - \mathbr{\mu}(s)-\mathbr{\beta}) + \frac{1}{2}\mathbr{\beta}^T \mathbr{\Sigma}^{-1}\mathbr{\beta}.
	\end{align*}
	Thus, we have:
	\begin{align*}
		\E_{\mathbr{a} \sim \pi_{\mathbr{\theta}}(\cdot|s)} \left[ \exp \left\{ \mathbr{\alpha}^T \mathbr{\overline{t}}(s,\mathbr{a}, \mathbr{\theta}) \right\} \right] & = \exp \left\{ \frac{1}{2}\mathbr{\beta}^T \mathbr{\Sigma}^{-1}\mathbr{\beta} \right\} \int_{\mathbb{R}^d} \frac{\exp\left\{ -\frac{1}{2} (\mathbr{a} - \mathbr{\mu}(s) - \mathbr{\beta})^T \mathbr{\Sigma}^{-1} (\mathbr{a} - \mathbr{\mu}(s) - \mathbr{\beta}) \right\}}{(2\pi)^{\frac{k}{2}} \det(\mathbr{\Sigma})^{\frac{1}{2}}}  \mathrm{d} \mathbr{a} \\
		& = \exp \left\{ \frac{1}{2}\mathbr{\beta}^T \mathbr{\Sigma}^{-1}\mathbr{\beta} \right\}.
	\end{align*}
	Now, we observe that:
	\begin{align*}
		\mathbr{\beta}^T \mathbr{\Sigma}^{-1}\mathbr{\beta} \le \| \mathbr{\beta} \|_2^2 \left\| \mathbr{\Sigma}^{-1} \right\|_2 \le \| \mathbr{\alpha} \|_2^2 \| \mathbr{\phi}(s) \|_2^2 \left\| \mathbr{\Sigma}^{-1} \right\|_2,
	\end{align*}
	having derived from Cauchy--Swartz inequality:
	 $$ \| \mathbr{\beta} \|_2^2  = \sum_{i=1}^k \left(\sum_{j=1}^q \alpha_{ij} \phi(s)_j \right)^2 \le \sum_{i=1}^k \sum_{j=1}^q \alpha_{ij}^2 \sum_{l=1}^q \phi(s)_l^2 = \left( \sum_{i=1}^k \sum_{j=1}^q \alpha_{ij}^2 \right)\sum_{l=1}^q \phi(s)_l^2  = \| \mathbr{\alpha} \|_2^2 \| \mathbr{\phi}(s) \|_2^2.$$
	We get the result by setting $\sigma = \Phi_{\max}\sqrt{ \left\| \mathbr{\Sigma}^{-1} \right\|_2} = \frac{\Phi_{\max}}{\sqrt{ \lambda_{\min} \left( \mathbr{\Sigma} \right)}}$.
\end{proof}

Furthermore, we report for completeness the standard Hoeffding concentration inequality for subgaussian random vectors.

\begin{prop}\label{prop:subGaussConcentration}
	Let $\mathbr{X}_1,\mathbr{X}_2, ..., \mathbr{X}_n$ be $n$ i.i.d. zero--mean subgaussian $d$--dimensional random vectors with parameter $\sigma \ge 0$, then for any $\mathbr{\alpha} \in \mathbb{R}^d$ and $\epsilon > 0$ it holds that:
	\begin{equation*}
		\Pr \left( \mathbr{\alpha}^T \left( \frac{1}{n} \sum_{i=1}^n \mathbr{X}_i \right)  \ge \epsilon \right) \le \exp\left\{ - \frac{\epsilon^2 n}{ 2 \left\| \mathbr{\alpha}\right\|_2^2 \sigma^2} \right\}.
	\end{equation*}
\end{prop}

\begin{proof}
	The proof is analogous to that of the Hoeffding inequality for bounded random variables. Let $s \ge 0$:
	\begin{align*}
		\Pr \left( \mathbr{\alpha}^T \left( \frac{1}{n} \sum_{i=1}^n \mathbr{X}_i \right)  \ge \epsilon \right) & = \Pr \left( \exp \left\{ s \mathbr{\alpha}^T \left( \frac{1}{n} \sum_{i=1}^n \mathbr{X}_i \right) \right\} \ge e^{s \epsilon}  \right) \\
		& \le e^{- s\epsilon} \E \left[ \exp \left\{ s \mathbr{\alpha}^T \left(\frac{1}{n} \sum_{i=1}^n \mathbr{X}_i \right) \right\} \right] \\
		& = e^{- s\epsilon} \prod_{i=1}^n \E \left[ \exp \left\{ \frac{s}{n} \mathbr{\alpha}^T \mathbr{X}_i  \right\} \right] \\
		& \le e^{- s\epsilon}  \exp \left\{ \frac{s^2}{2n} \left\| \mathbr{\alpha}\right\|_2^2 \sigma^2 \right\} = \exp \left\{- s\epsilon + \frac{s^2}{2n} \left\| \mathbr{\alpha}\right\|_2^2 \sigma^2 \right\}, \\
	\end{align*}
	where we employed Markov inequality, exploited the subgaussianity assumption and the independence. We minimize the last expression over $s$, getting the optimal $s = \frac{\epsilon n}{ \left\| \mathbr{\alpha}\right\|_2^2 \sigma^2}$, from which we get the result:
	\begin{align*}
	\Pr \left( \mathbr{\alpha}^T \left( \frac{1}{n} \sum_{i=1}^n \mathbr{X}_i \right)  \ge \epsilon \right) \le \exp\left\{ - \frac{\epsilon^2 n}{ 2 \left\| \mathbr{\alpha}\right\|_2^2 \sigma^2} \right\}.
	\end{align*}
\end{proof}

Under the Assumption~\ref{ass:subGauss}, we provide the following concentration inequality for the minimum eigenvalue of the empirical FIM.

\begin{prop}\label{prop:minEigFisher}
	Let $\mathcal{F}(\mathbr{\theta})$ and $\widehat{\mathcal{F}}(\mathbr{\theta})$ be the FIM and its estimate obtained with $n>0$ independent samples. Then, under Assumption~\ref{ass:subGauss}, for any $\epsilon > 0$ it holds that:
	\begin{equation*}
		\Pr \left( \left| {\lambda}_{\min}\left( \widehat{\mathcal{F}}(\mathbr{\theta}) \right) - {\lambda}_{\min}\left(\mathcal{F}(\mathbr{\theta}) \right) \right| > \epsilon \right) \le 2\exp \left\{ -\frac{\epsilon^2 n}{\psi_\sigma d^2 \sigma^4} \right\},
	\end{equation*}
	where $\psi_\sigma > 0$ is a constant depending only on the subgaussianity parameter $\sigma$.
	In particular, under the following condition on $n$ we have that, for any $\delta \in [0,1]$, ${\lambda}_{\min}( \widehat{\mathcal{F}}(\mathbr{\theta}) ) >0$ with probability at least $1-\delta$:
	\begin{equation*}
		n > \frac{d^2\sigma^4 \psi_\sigma \log \frac{2}{\delta}}{\lambda_{\min}(\mathcal{F}(\mathbr{\theta}))^2}.
	\end{equation*}
\end{prop}

\begin{proof}
	Let us recall that $\widehat{\mathcal{F}}(\mathbr{\theta})$ and $\mathcal{F}(\mathbr{\theta})$ are both symmetric positive semidefinite matrices, thus their eigenvalues $\lambda_j$ correspond to their singular values $\sigma_j$. Let us consider the following sequence of inequalities:
	\begin{align*}
		\left| {\lambda}_{\min}\left( \widehat{\mathcal{F}}(\mathbr{\theta}) \right) - {\lambda}_{\min}\left(\mathcal{F}(\mathbr{\theta}) \right) \right| & = \left| {\sigma}_{\min}\left( \widehat{\mathcal{F}}(\mathbr{\theta}) \right) - {\sigma}_{\min}\left(\mathcal{F}(\mathbr{\theta}) \right) \right| \\
		& \le \max_{j \in \interval} \left| {\sigma}_{j}\left( \widehat{\mathcal{F}}(\mathbr{\theta}) \right) - {\sigma}_{j}\left(\mathcal{F}(\mathbr{\theta}) \right) \right|\\
		& \le \left\|  \widehat{\mathcal{F}}(\mathbr{\theta}) - \mathcal{F}(\mathbr{\theta}) \right\|_2,
	\end{align*}
	where last inequality follows from~\citet{ben2003generalized}. Therefore, all it takes is to bound the norm of the difference. For this purpose, we employ Corollary 5.50 and Remark 5.51 of~\citet{vershynin2012introduction}, having observed that the FIM is indeed a covariance matrix and its estimate is a sample covariance matrix. We obtain that with probability at least $1-\delta$:
	\begin{equation}
		\left\| \widehat{\mathcal{F}}(\mathbr{\theta}) - \mathcal{F}(\mathbr{\theta}) \right\|_2 \le \left\| \mathcal{F}(\mathbr{\theta}) \right\|_2 \sqrt{\frac{\psi_\sigma \log \frac{2}{\delta}}{n}},
	\end{equation}
	where $\psi_\sigma \ge 0$ is a constant depending on the subgaussianity parameter $\sigma$. Recalling, from Lemma~\ref{lemma:fisherMaxEig}, that $\left\| \mathcal{F}(\mathbr{\theta}) \right\| = \lambda_{\max} \left( \mathcal{F}(\mathbr{\theta}) \right) \le d \sigma^2$, we can rewrite the previous inequality as:
	\begin{equation}
		\left\| \widehat{\mathcal{F}}(\mathbr{\theta}) - \mathcal{F}(\mathbr{\theta}) \right\|_2 \le d \sigma^2 \sqrt{\frac{\psi_\sigma \log \frac{2}{\delta}}{n}}.
	\end{equation}
	By setting the right hand side equal to $\epsilon$ and solving for $\delta$, we get the first result. The value of $n$ can be obtained by setting the right hand side equal to $\lambda_{\min}(\mathcal{F}(\mathbr{\theta}))$.
\end{proof}

\subsubsection{Concentration Result}\label{apx:concentration}
We are now ready to provide the main result of this section, that consists in a concentration result on the negative log--likelihood. Our final goal is to provide a probabilistic bound to the differences $\ell(\widehat{\mathbr{\theta}}) - \ell({\mathbr{\theta}}^*)$ and 	$\widehat{\ell}({\mathbr{\theta}}^*) - \widehat{\ell}(\widehat{\mathbr{\theta}})$. To this purpose, we start with a technical lemma (Lemma~\ref{lemma:gConcentration}) which provides a concentration result involving a quantity that will be used later, under Assumption~\ref{ass:subGauss}. Then, we use this result to obtain the concentration of the parameters, \ie bounding the distance $\left\| \widehat{ \mathbr{\theta}} - \mathbr{\theta}^* \right\|_2$ (Theorem~\ref{thr:paramConcentration}), under suitable well--conditioning properties of the involved quantities. Finally, we employ the latter result to prove the concentration of the negative log--likelihood (Corollary~\ref{coroll:likConcentration}). Some parts of the derivation are inspired to~\citet{li2017provably}.

\begin{lemma}\label{lemma:gConcentration}
	Under Assumption~\ref{ass:identifiability} and Assumption~\ref{ass:subGauss}, let $\mathcal{D} = \{(s_i,a_i)\}_{i=1}^n$ be a dataset of $n>0$ independent samples, where $s_i \sim d_{\mu}^{\pi_{\mathbr{\theta}^*}}$ and $a_i \sim \pi_{\mathbr{\theta}^*}(\cdot|s_i)$. For any $\mathbr{\theta} \in \Theta$, let $\mathbr{g}(\mathbr{\theta})$ be defined as:
	\begin{equation}
		\mathbr{g}(\mathbr{\theta}) = \frac{1}{n} \sum_{i=1}^n \left( \E_{a \sim \pi_{\mathbr{\theta}}(\cdot|s)} \left[ \mathbr{t}(s_i,a) \right] - \E_{a \sim \pi_{\mathbr{\theta}^*}(\cdot|s)} \left[ \mathbr{t}(s_i,a) \right]  \right).
	\end{equation}
	Let $\widehat{\mathbr{\theta}} = \argmin_{\mathbr{\theta} \in \Theta} \widehat{\ell}(\mathbr{\theta}) = \frac{1}{n} \sum_{i=1}^n \log \pi_{\mathbr{\theta}}(a_i|s_i)$. Then, under Assumption~\ref{ass:subGauss}, for any $\delta \in [0,1]$, with probability at least $1-\delta$, it holds that:
	\begin{equation}
		\left\| \mathbr{g}(\widehat{\mathbr{\theta}}) \right\|_2 \le \sigma \sqrt{\frac{2d}{n} \log \frac{2d}{\delta}}.
	\end{equation}
\end{lemma}

\begin{proof}
	The negative log--likelihood of a policy complying with Definition~\ref{defi:expFamily} is $\mathcal{C}^2 (\mathbb{R}^d)$. Thus, since $\widehat{\mathbr{\theta}}$ is a minimizer of the negative log--likelihood function $\widehat{\ell}(\mathbr{\theta})$, it must fulfill the following first--order condition:
	\begin{equation}
		\nabla_{\mathbr{\theta}} \widehat{\ell}(\widehat{\mathbr{\theta}}) = \frac{1}{n} \sum_{i=1}^n \nabla_{\mathbr{\theta}} \log \pi_{\widehat{\mathbr{\theta}}}(a_i|s_i) = \frac{1}{n} \sum_{i=1}^n \left( \mathbr{t}(s_i,a_i) - \E_{{a} \sim \pi_{\widehat{\mathbr{\theta}}} (\cdot|s)} \left[\mathbr{t}(s_i,{a}) \right] \right) = \mathbr{0}. 
	\end{equation}
	As a consequence, we can rewrite the expression of $\mathbr{g}(\widehat{\mathbr{\theta}})$ exploiting this condition:
	\begin{equation}
		\mathbr{g}(\widehat{\mathbr{\theta}}) = \frac{1}{n} \sum_{i=1}^n \left( \E_{a \sim \pi_{\widehat{\mathbr{\theta}}}(\cdot|s)} \left[ \mathbr{t}(s_i,a) \right] - \E_{a \sim \pi_{\mathbr{\theta}^*}(\cdot|s)} \left[ \mathbr{t}(s_i,a) \right]  \right) = \frac{1}{n} \sum_{i=1}^n \left(\mathbr{t}(s_i,a_i)  - \E_{a \sim \pi_{\mathbr{\theta}^*}(\cdot|s)} \left[ \mathbr{t}(s_i,a) \right]  \right) = \frac{1}{n} \sum_{i=1}^n \mathbr{\overline{t}}(s_i,a_i,\mathbr{\theta}^*).
	\end{equation}
	By recalling that $a_i \sim \pi_{\mathbr{\theta}^*} (\cdot|s_i)$ it immediately follows that $\mathbr{g}(\widehat{\mathbr{\theta}})$ is a zero-mean random vector, \ie 
	$\E_{\substack{s_i \sim d_{\mu}^{\pi_{\mathbr{\theta}^*}} \\ a_i \sim \pi_{\mathbr{\theta}^*}(\cdot|s_i) } } \left[ \mathbr{g}(\widehat{\mathbr{\theta}}) \right] = \mathbr{0}$. Moreover, under Assumption~\ref{ass:subGauss}, $\mathbr{g}(\widehat{\mathbr{\theta}})$ is the sample mean of subgaussian random vectors. Our goal is to bound the probability $\Pr \left(\left\| \mathbr{g}(\widehat{\mathbr{\theta}}) \right\|_2 > \epsilon  \right)$; to this purpose we consider the following derivation:
	\begin{align}
		\Pr \left(\left\| \mathbr{g}(\widehat{\mathbr{\theta}}) \right\|_2 > \epsilon  \right) & = \Pr \left( \sqrt{\sum_{j=1}^d g_j(\widehat{\mathbr{\theta}})^2} > \epsilon  \right) \notag \\
		& \le \Pr \left( \bigvee_{j=1}^d \left| g_j(\widehat{\mathbr{\theta}}) \right| > \frac{\epsilon}{\sqrt{d}}  \right) \label{p:01}\\
		& \le \sum_{j=1}^d \Pr \left( \left| g_j(\widehat{\mathbr{\theta}}) \right| > \frac{\epsilon}{\sqrt{d}}  \right), \label{p:02}
	\end{align}
	where we exploited in line~\eqref{p:01} the fact that for a $d$-dimensional vector $\mathbr{x}$ if $\left\| \mathbr{x} \right\|_2 > \epsilon$ it must be that at least one component $j=1,...,d$ satisfy $x_j^2 > \frac{\epsilon^2}{d}$ and we used a union bound over the $d$ dimensions to get line~\eqref{p:02}. Since for each $j=1,...,d$  we have that $ g_j(\widehat{\mathbr{\theta}}) $ is a zero-mean subgaussian random variable we can bound the deviation using standard results~\cite{boucheron2013concentration}:
	\begin{equation}
		\Pr \left( \left| g_j(\widehat{\mathbr{\theta}}) \right| > \frac{\epsilon}{\sqrt{d}}  \right) \le 2 \exp \left\{ - \frac{\epsilon^2 n}{2 d \sigma^2 } \right\}.
	\end{equation}
	Putting all together we get:
	\begin{equation}
		\Pr \left(\left\| \mathbr{g}(\widehat{\mathbr{\theta}}) \right\|_2 > \epsilon  \right) \le 2 d \exp \left\{ - \frac{\epsilon^2 n}{2 d \sigma^2 } \right\}.
	\end{equation}
	By setting $\delta = 2 d \exp \left\{ - \frac{\epsilon^2 n}{2 d \sigma^2 } \right\}$ and solving for $\epsilon$ we get the result.

\end{proof}

We can now use the previous result to derive the concentration of the parameters, \ie bounding the deviation $\left\| \widehat{ \mathbr{\theta}} - \mathbr{\theta}^* \right\|_2$.
\begin{restatable}[Parameter concentration]{thr}{paramConcentration}\label{thr:paramConcentration}
Under Assumption~\ref{ass:identifiability} and Assumption~\ref{ass:subGauss}, let $\mathcal{D} = \{(s_i,a_i)\}_{i=1}^n$ be a dataset of $n>0$ independent samples, where $s_i \sim d_{\mu}^{\pi_{\mathbr{\theta}^*}}$ and $a_i \sim \pi_{\mathbr{\theta}^*}(\cdot|s_i)$. Let $\widehat{\mathbr{\theta}} = \argmin_{\mathbr{\theta} \in \Theta} \widehat{\ell}(\mathbr{\theta})$. If the empirical FIM $\widehat{\mathcal{F}}(\mathbr{\theta})$ has a positive minimum eigenvalue $\widehat{\lambda}_{\min} > 0$ for all $\mathbr{\theta} \in \Theta$, for any $\delta \in [0,1]$, with probability at least $1-\delta$, it holds that:
	\begin{equation}
		\left\| \widehat{ \mathbr{\theta}} - \mathbr{\theta}^* \right\|_2 \le \frac{\sigma}{\widehat{\lambda}_{\min}} \sqrt{\frac{2d}{n} \log \frac{2d}{\delta}}.
	\end{equation}
\end{restatable}

\begin{proof}
	Recalling that $\mathbr{g}({\mathbr{\theta}}^*) =\mathbr{0}$, we employ the mean value theorem to rewrite $\mathbr{g}(\widehat{\mathbr{\theta}})$ centered in ${\mathbr{\theta}}^*$:
	\begin{equation}
		 \mathbr{g}(\widehat{\mathbr{\theta}})  = \mathbr{g}(\widehat{\mathbr{\theta}}) - \mathbr{g}({\mathbr{\theta}}^*) = \widehat{\mathcal{F}}(\overline{\mathbr{\theta}}) \left(\widehat{\mathbr{\theta}} -{\mathbr{\theta}}^*  \right),
	\end{equation}
	where $\overline{\mathbr{\theta}} = t \widehat{\mathbr{\theta}} + (1-t) {\mathbr{\theta}}^* $ for some $t \in [0,1]$ and $\widehat{\mathcal{F}}(\overline{\mathbr{\theta}})$ is defined as:
	\begin{align*}
	\widehat{\mathcal{F}}({\mathbr{\theta}}) = \nabla_{\mathbr{\theta}} \mathbr{g}(\mathbr{\theta}) & = \frac{1}{n} \sum_{i=1}^n \E_{a \sim \pi_{\mathbr{\theta}}(\cdot|s)} \left[\nabla_{\mathbr{\theta}} \log  \pi_{\mathbr{\theta}}(a|s) \mathbr{t}(s_i,a) \right]  \\
		& = \frac{1}{n} \sum_{i=1}^n \E_{a \sim \pi_{\mathbr{\theta}}(\cdot|s)} \left[\left(\mathbr{t}(s_i,a) - \E_{\overline{a} \sim \pi_{\mathbr{\theta}}(\cdot|s)} \left[ \mathbr{t}(s_i,\overline{a}) \right] \right) \mathbr{t}(s_i,a) \right] = \widehat{\mathcal{F}}(\mathbr{\theta}),
	\end{align*}
	where we exploited the expression of $\nabla_{\mathbr{\theta}} \log  \pi_{\mathbr{\theta}}(a|s) $ and the definition of Fisher information matrix given in Equation~\eqref{eq:fisher}. Under the hypothesis of the statement, we can derive the following lower bound:
	\begin{equation}
		\left\| \mathbr{g}(\widehat{\mathbr{\theta}}) \right\|_2^2 =   \left(\widehat{\mathbr{\theta}} -{\mathbr{\theta}}^*  \right)^T \widehat{\mathcal{F}}(\overline{\mathbr{\theta}})^T \widehat{\mathcal{F}}(\overline{\mathbr{\theta}}) \left(\widehat{\mathbr{\theta}} -{\mathbr{\theta}}^*  \right) \ge  \widehat{\lambda}_{\min}^2 \left\|\widehat{\mathbr{\theta}} -{\mathbr{\theta}}^* \right\|_2^2.
	\end{equation}
	By solving for $\left\|\widehat{\mathbr{\theta}} -{\mathbr{\theta}}^* \right\|_2$ and applying Lemma~\ref{lemma:gConcentration} we get the result.
\end{proof}

Finally, we can get the concentration result for the negative log--likelihood.

\begin{restatable}[Negative log--likelihood concentration]{coroll}{}\label{coroll:likConcentration}
Under Assumption~\ref{ass:identifiability} and Assumption~\ref{ass:subGauss}, let $\mathcal{D} = \{(s_i,a_i)\}_{i=1}^n$ be a dataset of $n>0$ independent samples, where $s_i \sim d_{\mu}^{\pi_{\mathbr{\theta}^*}}$ and $a_i \sim \pi_{\mathbr{\theta}^*}(\cdot|s_i)$. Let $\widehat{\mathbr{\theta}} = \argmin_{\mathbr{\theta} \in \Theta} \widehat{\ell}(\mathbr{\theta})$. If ${\lambda}_{\min}(\widehat{\mathcal{F}}({\mathbr{\theta}})) = \widehat{\lambda}_{\min} > 0$ for all $\mathbr{\theta} \in \Theta$, for any $\delta \in [0,1]$, with probability at least $1-\delta$, it holds that:
	\begin{equation}
		\ell(\widehat{\mathbr{\theta}}) - \ell({\mathbr{\theta}}^*) \le \frac{d^2\sigma^4}{\widehat{\lambda}_{\min}^2 n}  \log \frac{2d}{\delta},
	\end{equation}
	and also:
	\begin{equation}
		\widehat{\ell}({\mathbr{\theta}}^*) - \widehat{\ell}(\widehat{\mathbr{\theta}})  \le \frac{ d^2\sigma^4}{\widehat{\lambda}_{\min}^2 n}  \log \frac{2d}{\delta}.
	\end{equation}
\end{restatable}

\begin{proof}
	Let us start with $\ell(\widehat{\mathbr{\theta}}) - \ell({\mathbr{\theta}}^*)$. We consider the first order Taylor expansion of the negative log--likelihood centered in ${\mathbr{\theta}}^*$:
	\begin{equation}
		\ell(\widehat{\mathbr{\theta}}) - \ell({\mathbr{\theta}}^*) = \nabla_{\mathbr{\theta}} \ell({\mathbr{\theta}}^*) ^T \left(\widehat{\mathbr{\theta}} -  {\mathbr{\theta}}^*\right) +  \frac{1}{2} \left(\widehat{\mathbr{\theta}} -  {\mathbr{\theta}}^*\right)^T  \mathcal{H}_{\mathbr{\theta}} \ell(\overline{\mathbr{\theta}}) \left(\widehat{\mathbr{\theta}} -  {\mathbr{\theta}}^*\right),
	\end{equation}
	where $\overline{\mathbr{\theta}} = t \widehat{\mathbr{\theta}} + (1-t) {\mathbr{\theta}}^* $ for some $t \in [0,1]$. 	We first observe that $\nabla_{\mathbr{\theta}} \ell({\mathbr{\theta}}^*) = \mathbr{0}$ being ${\mathbr{\theta}}^*$ the true parameter and we develop $\mathcal{H}_{\mathbr{\theta}} \ell(\overline{\mathbr{\theta}})$:
	\begin{align*}
		\mathcal{H}_{\mathbr{\theta}} \ell(\overline{\mathbr{\theta}}) & = \E_{\substack{s \sim d_{\mu}^{\pi_{\mathbr{\theta}^*}} \\ a \sim \pi_{\mathbr{\theta}^*}(\cdot|s)}} \left[ \mathcal{H}_{\mathbr{\theta}} \log \pi_{\overline{\mathbr{\theta}}}(a|s) \right] \\
		& = \E_{\substack{s \sim d_{\mu}^{\pi_{\mathbr{\theta}^*}} \\ a \sim \pi_{\mathbr{\theta}^*}(\cdot|s)}} \left[ \nabla_{\mathbr{\theta}} \left( \mathbr{t}(s,a) - \E_{\overline{a} \sim \pi_{\overline{\mathbr{\theta}}} (\cdot|s)} \left[\mathbr{t}(s,\overline{a}) \right] \right) \right] \\
		& = \E_{s \sim d_{\mu}^{\pi_{\mathbr{\theta}^*}}} \left[ \nabla_{\mathbr{\theta}} \E_{\overline{a} \sim \pi_{\overline{\mathbr{\theta}}} (\cdot|s)} \left[\mathbr{t}(s,\overline{a}) \right] \right] \\
		& = \E_{s \sim d_{\mu}^{\pi_{\mathbr{\theta}^*}}} \left[ \E_{\overline{a} \sim \pi_{\overline{\mathbr{\theta}}} (\cdot|s)} \left[ \left( \mathbr{t}(s,\overline{a}) - \E_{\widetilde{a} \sim \pi_{\overline{\mathbr{\theta}}} (\cdot|s)} \left[\mathbr{t}(s,\widetilde{a}) \right] \right) \mathbr{t}(s,\overline{a})^T \right] \right] =  \E_{s \sim d_{\mu}^{\pi_{\mathbr{\theta}^*}}}  \left[ \mathcal{F}(\overline{\mathbr{\theta}},s) \right].
	\end{align*}
	By using Lemma~\ref{lemma:fisherMaxEig} to bound the maximum eigenvalue of $\mathcal{F}(\overline{\mathbr{\theta}},s)$, we can state the inequality:
	\begin{equation}
		\frac{1}{2} \left(\widehat{\mathbr{\theta}} -  {\mathbr{\theta}}^*\right)^T  \mathcal{H}_{\mathbr{\theta}} \ell(\overline{\mathbr{\theta}}) \left(\widehat{\mathbr{\theta}} -  {\mathbr{\theta}}^*\right) \le \frac{d \sigma^2}{2} \left\|\widehat{\mathbr{\theta}} -  {\mathbr{\theta}}^*\right\|_2^2. 
	\end{equation}
	Using the concentration result of Theorem~\ref{thr:paramConcentration}, we get the result. Concerning $\widehat{\ell}({\mathbr{\theta}}^*) - \widehat{\ell}(\widehat{\mathbr{\theta}})$, the derivation is analogous with the only difference that the Taylor expansion has to be centered in $\widehat{\mathbr{\theta}}$ instead of ${\mathbr{\theta}}^*$.
\end{proof}

To conclude this appendix, we present the following technical lemma.

\begin{thr}\label{thr:otherConcentration}
Under Assumption~\ref{ass:identifiability} and Assumption~\ref{ass:subGauss}, let $\mathcal{D} = \{(s_i,a_i)\}_{i=1}^n$ be a dataset of $n>0$ independent samples, where $s_i \sim d_{\mu}^{\pi_{\mathbr{\theta}^*}}$ and $a_i \sim \pi_{\mathbr{\theta}^*}(\cdot|s_i)$. Let $\mathbr{\theta}, \mathbr{\theta}' \in \Theta$, then for any $\epsilon > 0$, it holds that:
	\begin{equation*}
		\Pr \left( \left[ \ell(\mathbr{\theta}) - \widehat{\ell}(\mathbr{\theta}) \right] - \left[ \ell(\mathbr{\theta}') - \widehat{\ell}(\mathbr{\theta}') \right]  > \epsilon \right) \le \exp \left\{ - \frac{\epsilon^2 n}{2 \left\| \mathbr{\theta} - \mathbr{\theta}' \right\|_2^2 \sigma^2 } \right\}.
	\end{equation*}
\end{thr}

\begin{proof}
	We write explicitly the involved expression, using Definition~\ref{defi:expFamily} and perform some algebraic manipulations:
	\begin{align*}
		\left[ \ell(\mathbr{\theta}) - \widehat{\ell}(\mathbr{\theta}) \right] & - \left[ \ell(\mathbr{\theta}') - \widehat{\ell}(\mathbr{\theta}') \right] = \E_{\substack{s \sim d_{\mu}^{\pi_{\mathbr{\theta}^*}} \\ a \sim \pi_{\mathbr{\theta}^*}(\cdot|s)}} \left[ \mathbr{\theta}^T  \mathbr{t}(s,a) - A(\mathbr{\theta},s) \right] - \frac{1}{n} \sum_{i=1}^n  \left( \mathbr{\theta}^T  \mathbr{t}(s_i,a_i) -A(\mathbr{\theta},s_i) \right) \\
		& \quad - \E_{\substack{s \sim d_{\mu}^{\pi_{\mathbr{\theta}^*}} \\ a \sim \pi_{\mathbr{\theta}^*}(\cdot|s)}} \left[ \left(\mathbr{\theta}'\right)^T  \mathbr{t}(s,a) -A(\mathbr{\theta}',s) \right] + \frac{1}{n} \sum_{i=1}^n  \left( \left(\mathbr{\theta}'\right)^T  \mathbr{t}(s_i,a_i) - A(\mathbr{\theta}',s_i) \right) \\
		& = \E_{\substack{s \sim d_{\mu}^{\pi_{\mathbr{\theta}^*}} \\ a \sim \pi_{\mathbr{\theta}^*}(\cdot|s)}} \left[\left(\mathbr{\theta} - \mathbr{\theta}'\right)^T \mathbr{t}(s,a) - \left( A(\mathbr{\theta},s) - A(\mathbr{\theta}',s) \right) \right] - \frac{1}{n} \sum_{i=1}^n \left(\left(\mathbr{\theta} - \mathbr{\theta}'\right)^T \mathbr{t}(s_i,a_i) - \left( A(\mathbr{\theta},s_i) - A(\mathbr{\theta}',s_i) \right) \right).
	\end{align*}
	Essentially, we are comparing the mean and the sample mean of the random variable $\left(\mathbr{\theta} - \mathbr{\theta}'\right)^T \mathbr{t}(s,a) - \left( A(\mathbr{\theta},s) - A(\mathbr{\theta}',s) \right)$. Let us now focus on $A(\mathbr{\theta},s) - A(\mathbr{\theta}',s)$. From the mean value theorem we know that, for some $t \in [0,1]$ and $\overline{\mathbr{\theta}} = t\mathbr{\theta} + (1-t)\mathbr{\theta}'$, we have:
	\begin{equation}
		A(\mathbr{\theta},s) - A(\mathbr{\theta}',s) = \nabla_{\mathbr{\theta}} A(\overline{\mathbr{\theta}},s) ^T \left( \mathbr{\theta}-\mathbr{\theta}' \right).
	\end{equation}
	From Equation~\eqref{eq:grad-log-pi}, we know that $\nabla_{\mathbr{\theta}} A(\overline{\mathbr{\theta}},s) = \E_{\overline{a} \sim \pi_{\overline{\mathbr{\theta}}}(\cdot|s)} \left[ \mathbr{t}(s,\overline{a}) \right]$. The random variable $\mathbr{\overline{t}}(s,a,\overline{\mathbr{\theta}}) =  \mathbr{t}(s,a) - \E_{\overline{a} \sim \pi_{\overline{\mathbr{\theta}}}(\cdot|s)} \left[ \mathbr{t}(s,\overline{a}) \right]$ is a subgaussian random variable for any $\overline{\mathbr{\theta}} \in \Theta$. Thus, under Assumption~\ref{ass:subGauss} we have:
	\begin{align*}
	\left[ \ell(\mathbr{\theta}) - \widehat{\ell}(\mathbr{\theta}) \right] & - \left[ \ell(\mathbr{\theta}') - \widehat{\ell}(\mathbr{\theta}') \right] = \left( \mathbr{\theta}-\mathbr{\theta}' \right)^T \left( \E_{\substack{s \sim d_{\mu}^{\pi_{\mathbr{\theta}^*}} \\ a \sim \pi_{\mathbr{\theta}^*}(\cdot|s)}} \left[  \mathbr{\overline{t}}(s,a,\overline{\mathbr{\theta}}) \right] - \frac{1}{n} \sum_{i=1}^n  \mathbr{\overline{t}}(s_i,a_i,\overline{\mathbr{\theta}}) \right).
	\end{align*}
	If we apply Proposition~\ref{prop:subGaussConcentration}, we get the result.
\end{proof}

\subsection{Results on Significance and Power of the Tests}

\sigPowerSimp*
\begin{proof}
We can rewrite the definitions of $\alpha$ and $\beta$ in the following form:
\begin{align*}
	& \alpha = \frac{1}{d-d^*} \E \left[ \left| \left\{ i \notin I^* : i \in \widehat{I}_{c} \right\} \right| \right] = \frac{1}{d-d^*} \sum_{i \notin I^*}  \Pr\left( i \in \widehat{I}_{c} \right) \\
	& \beta = \frac{1}{d^*} \E \left[ \left| \left\{ i \in I^* : i \notin \widehat{I}_{c} \right\} \right| \right] = \frac{1}{d^*} \sum_{i \in I^*}  \Pr\left( i \notin \widehat{I}_{c} \right).
\end{align*}
Now, we focus on bounding $\Pr\left( i \in \widehat{I}_{c} \right)$. We employ an argument analogous to that of~\cite{garivier2019non}:
	\begin{align*}
		\Pr \left( i \in \widehat{I}_{c} \right)  & = \Pr \left( \lambda_i > c(1) \right) \\
		& = \Pr \left( \widehat{\ell}(\widehat{\mathbr{\theta}_i}) - \widehat{\ell}(\widehat{\mathbr{\theta}}) > \frac{c(1)}{2} \right) \\
		& \le \Pr \left( \widehat{\ell}(\mathbr{\theta}^*) - \widehat{\ell}(\widehat{\mathbr{\theta}}) > \frac{c(1)}{2} \right) \le 2d \exp \left\{ -\frac{c(\delta,1) \lambda_{\min}^2 n}{16d^2 \sigma^4} \right\},
	\end{align*}
	where we observed that $\widehat{\ell}(\mathbr{\theta}^*) \ge \widehat{\ell}(\widehat{\mathbr{\theta}_i})$ as $\mathbr{\theta}^* \in \Theta_i$ under $\mathcal{H}_0$ and we applied Corollary~\ref{coroll:likConcentration} in the last line, recalling that $\widehat{\lambda}_{\min} \ge \frac{\lambda_{\min}}{2\sqrt{2}}$. For the second inequality, the derivation is a little more articulated. 
		Let us now focus on the single terms $\Pr\left( i \notin \widehat{I}_{c} \right)$. We now perform the following manipulations:
		\begin{align}
		\Pr\left( i \notin \widehat{I}_{c} \right) & = \Pr \left(\widehat{\ell}(\widehat{\mathbr{\theta}_i}) - \widehat{\ell}(\widehat{\mathbr{\theta}}) \le \frac{c(1)}{2} \right) \notag \\
		& = \Pr \left(\left[ \widehat{\ell}(\widehat{\mathbr{\theta}_i}) - \widehat{\ell}({\mathbr{\theta}_i^*}) \right] + \left[ \widehat{\ell}({\mathbr{\theta}}^*) - \widehat{\ell}(\widehat{\mathbr{\theta}}) \right] + \left[ \widehat{\ell}({\mathbr{\theta}_i^*}) - \widehat{\ell}({\mathbr{\theta}^*} ) \right] \le \frac{c(1)}{2} \right) \\
		& \le \Pr \left(\left[ \widehat{\ell}(\widehat{\mathbr{\theta}_i}) - \widehat{\ell}({\mathbr{\theta}_i^*}) \right] + \left[ \widehat{\ell}({\mathbr{\theta}_i}^*) - \widehat{\ell}({\mathbr{\theta}}^*) \right] \le \frac{c(1)}{2} \right) \label{p:51}\\
		& = \Pr \left(\left[ \widehat{\ell}(\widehat{\mathbr{\theta}_i}) - \widehat{\ell}({\mathbr{\theta}_i^*}) \right] + \left[\widehat{\ell}({\mathbr{\theta}_i^*}) -{\ell}({\mathbr{\theta}_i^*}) \right] + \left[ {\ell}({\mathbr{\theta}^*}) - \widehat{\ell}({\mathbr{\theta}^*}) \right] \le \frac{c(1)}{2} + \left[ {\ell}({\mathbr{\theta}^*}) - {\ell}({\mathbr{\theta}_i^*}) \right] \right) \notag \\
		& =  \Pr \left(\left[\widehat{\ell}({\mathbr{\theta}_i^*}) - \widehat{\ell}(\widehat{\mathbr{\theta}_i})  \right] + \left[{\ell}({\mathbr{\theta}_i^*}) - \widehat{\ell}({\mathbr{\theta}_i^*})\right] + \left[ \widehat{\ell}({\mathbr{\theta}^*})-{\ell}({\mathbr{\theta}^*}) \right] \ge \left[ {\ell}({\mathbr{\theta}_i^*}) - {\ell}({\mathbr{\theta}^*})  \right]  - \frac{c(1)}{2} \right).\notag 
		\end{align}
		where line~\eqref{p:51} is obtained by observing that $ \widehat{\ell}({\mathbr{\theta}^*}) - \widehat{\ell}(\widehat{\mathbr{\theta}}) \ge 0$. Thus, we have:
		\begin{align}
		 \Pr\left( i \notin \widehat{I}_{c} \right) & \le \Pr \left(\widehat{\ell}({\mathbr{\theta}_i}^*) - \widehat{\ell}(\widehat{\mathbr{\theta}_i}) \ge \frac{1}{2} \left[ {\ell}({\mathbr{\theta}_i}^*) - {\ell}({\mathbr{\theta}}^*)  \right]  - \frac{c(1)}{2} \right) \notag \\
		 & \quad + \Pr \left(\left[{\ell}({\mathbr{\theta}_i}^*) - \widehat{\ell}({\mathbr{\theta}_i}^*)\right] + \left[ \widehat{\ell}({\mathbr{\theta}}^*)-{\ell}({\mathbr{\theta}}^*) \right] \ge \frac{1}{2} \left[ {\ell}({\mathbr{\theta}_i}^*) - {\ell}({\mathbr{\theta}}^*)  \right]  \right) \label{p:52}\\
		 & \le \Pr \left(\widehat{\ell}({\mathbr{\theta}_i}^*) - \widehat{\ell}(\widehat{\mathbr{\theta}_i}) \ge \frac{1}{2} \left[ {\ell}({\mathbr{\theta}_i}^*) - {\ell}({\mathbr{\theta}}^*)  \right]  - \frac{c(1)}{2} \right) \notag \\
		 & \quad + \Pr \left(\left[{\ell}({\mathbr{\theta}_i}^*) - \widehat{\ell}({\mathbr{\theta}_i}^*)\right] + \left[ \widehat{\ell}({\mathbr{\theta}}^*)-{\ell}({\mathbr{\theta}}^*) \right] \ge \frac{1}{2} \left[ \frac{1}{2} \lambda_{\min} \left( {\ell}({\mathbr{\theta}_i}^*) - {\ell}({\mathbr{\theta}}^*) \right) \left\| \mathbr{\theta}_i^* - \mathbr{\theta}^* \right\|_2^2  \right]^{\frac{1}{2}}  \right) \label{p:53}\\
		 & \le 2(d-1) \exp\left\{ - \frac{\left( {\ell}({\mathbr{\theta}_i^*}) - {\ell}({\mathbr{\theta}}^*)  -c(1) \right) \lambda_{\min}^2 n}{16 (d-1)^2 \sigma^4} \right\} + \exp \left\{ - \frac{ \left({\ell}({\mathbr{\theta}_i}^*) - {\ell}({\mathbr{\theta}}^*) \right) \lambda_{\min}  n}{16  \sigma^2} \right\} \label{p:54}\\
		 & \le 2(d-1) \exp\left\{ - \frac{\left( {\ell}({\mathbr{\theta}_i}^*) - {\ell}({\mathbr{\theta}}^*)  -c(1) \right) \lambda_{\min} n
		 \nu}{16 (d-1)^2 \sigma^2} \right\} + \exp \left\{ - \frac{ \left( {\ell}({\mathbr{\theta}_i}^*) - {\ell}({\mathbr{\theta}}^*)  -c(1) \right) \lambda_{\min} n \nu}{16 (d-1)^2 \sigma^2} \right\} \label{p:55}\\
		 & \le (2d - 1) \exp\left\{ - \frac{\left( {\ell}({\mathbr{\theta}_i}^*) - {\ell}({\mathbr{\theta}}^*)  -c(1) \right) \lambda_{\min} n
		 \nu}{16 (d-1)^2 \sigma^2} \right\}.\notag
		\end{align}
		where line~\eqref{p:52} derives from the inequality $\Pr(X+Y \ge c) \le \Pr(X \ge a) + \Pr(Y \ge b)$ with $c=a+b$, line~\eqref{p:53} is obtained by the following second order Taylor expansion, recalling that $\nabla_{\mathbr{\theta}}  {\ell}({\mathbr{\theta}}^*) = \mathbr{0}$:
		\begin{align*}
			{\ell}({\mathbr{\theta}_i}^*) - {\ell}({\mathbr{\theta}}^*)  = \nabla_{\mathbr{\theta}}  {\ell}({\mathbr{\theta}}^*)^T \left({\mathbr{\theta}_i}^* - {\mathbr{\theta}}^*\right)  + \frac{1}{2} \left({\mathbr{\theta}_i}^* - {\mathbr{\theta}}^*\right)^T \mathcal{H}_{\mathbr{\theta}} {\ell}(\overline{\mathbr{\theta}}) \left({\mathbr{\theta}_i}^* - {\mathbr{\theta}}^*\right) \ge \frac{\lambda_{\min}}{2} \left\|{\mathbr{\theta}_i}^* - {\mathbr{\theta}}^* \right\|_2^2,
\end{align*}		
where  $\overline{\mathbr{\theta}} = t \mathbr{\theta}^* + (1-t) \mathbr{\theta}^*_i$ for some $t \in [0,1]$. Line~\eqref{p:54} is obtained by applying Corollary~\ref{coroll:likConcentration}, recalling that $\widehat{\lambda}_{\min} \ge \frac{\lambda_{\min}}{2\sqrt{2}}$ and Theorem~\ref{thr:otherConcentration}. Finally, line~\eqref{p:55} derives by introducing the term $\nu = \min\left\{ 1, \frac{\lambda_{\min}}{\sigma^2} \right\}$ and observing that:
\begin{equation*}
	\frac{\left( {\ell}({\mathbr{\theta}_i}^*) - {\ell}({\mathbr{\theta}}^*)  -c(1) \right)
		 \nu}{(d-1)^2} \le \frac{ \left({\ell}({\mathbr{\theta}_i}^*) - {\ell}({\mathbr{\theta}}^*) \right)  n}{16}.
\end{equation*}
Clearly, this result is meaningful as long as $ {\ell}({\mathbr{\theta}_i}^*) - {\ell}({\mathbr{\theta}}^*)   - c(1) \ge 0$.

In order to derive the value of $\delta$ for the $\delta$--correctness, we consider the following derivation:
\begin{align*}
	\Pr \left( \widehat{I}_c \neq I^* \right) &= \Pr \left(\exists i \notin I^* : i \in \widehat{I}_c \vee \exists i \in I^* : i \notin \widehat{I}_c\right) \\
	& \le  \Pr \left(\exists i \notin I^* : i \in \widehat{I}_c \right) + \Pr \left(\exists i \in I^* : i \notin \widehat{I}_c \right)  \\
	& \le \sum_{i \notin  I^*} \Pr \left(i \in \widehat{I}_c \right) + \sum_{i \in I^*} \Pr \left(i \notin \widehat{I}_c \right)  \\
	& = (d-d^*) \alpha + d^* \beta,
\end{align*}
where we employed union bound in the first two passages and the definition of $\alpha$ and $\beta$ in the last one.
%
\end{proof}

\subsection{Additional Proofs and Derivations}\label{apx:additional}

\lemmaRuleIdent*
\begin{proof}
	The uniqueness of $I^*$ is ensured by Assumption~\ref{ass:identifiability}. Let us rewrite the condition of Definition~\ref{eq:problemOriginal} under Assumption~\ref{ass:identifiability}:
	\begin{align}
		\pi^* \in \Pi_{\Theta_{I^*}} & \, \wedge \,\forall i \in {I^*} : \pi^* \notin \Pi_{\Theta_{I^* \setminus \{i\}}} \iff \mathbr{\theta}^* \in \Theta_{I^*} \, \wedge \,\forall i \in {I^*} : \mathbr{\theta}^* \notin \Theta_{I^* \setminus \{i\}}\label{eq:p31} \\
		& \iff \forall i \in I^*: \theta^*_i \neq 0 \, \wedge \,\forall i \in \interval \setminus I^* : \theta_i^* =0 \label{eq:p32}\\
		& \iff I^* = \left\{ i \in \interval \,:\, \theta_i^* \neq 0 \right\},\notag
	\end{align}
	where line~\eqref{eq:p31} follows since there is a unique representation for $\pi^*$ determined by parameter $\mathbr{\theta}^*$ and line~\eqref{eq:p32} is obtained from the definition of $\Theta_I$.
\end{proof}

\lowerBoundObjective*
\begin{proof}
	By second order Taylor expansion of $\ell$ and recalling that $\nabla_{\mathbr{\theta}}  {\ell}({\mathbr{\theta}^*(\mathbr{\omega})}) = \mathbr{0}$, we have:
	\begin{align*}
		{\ell}({\mathbr{\theta}_I^*}(\mathbr{\omega})) - {\ell}({\mathbr{\theta}^*}(\mathbr{\omega})) & \ge \frac{\lambda_{\min}}{2} \left\|{\mathbr{\theta}_I^*}(\mathbr{\omega}) - {\mathbr{\theta}^*}(\mathbr{\omega}) \right\|_2^2\\
		& = \frac{\lambda_{\min}}{2} \left\|\mathbr{\theta}^*(\mathbr{\omega}_0) + \alpha \nabla_{\mathbr{\theta}} J_{\mathcal{M}_{\mathbr{\omega}}} (\mathbr{\theta}^*(\mathbr{\omega}_0)) \rvert_I - \mathbr{\theta}^*(\mathbr{\omega}_0) - \alpha \nabla_{\mathbr{\theta}} J_{\mathcal{M}_{\mathbr{\omega}}} (\mathbr{\theta}^*(\mathbr{\omega}_0)) \right\|_2^2\\
		& = \frac{\lambda_{\min}\alpha^2}{2} \left\| \nabla_{\mathbr{\theta}} J_{\mathcal{M}_{\mathbr{\omega}}} (\mathbr{\theta}^*(\mathbr{\omega}_0)) \rvert_{\overline{I}}   \right\|_2^2.
	\end{align*}
\end{proof}

\section{Detail on Identification Rules with Configurable Environment}\label{apx:configuration}
In this appendix, we report some details about the application of the environment configurability to the problem of policy space identification. We assume to have access to a dataset of trajectories $\mathcal{D} = \{\tau_i\}_{i=1}^n$ independently collected using policy $\pi_{\mathbr{\theta}}$ in the environment $\mathcal{M}_{\mathbr{\omega}_0}$. Each trajectory is a sequence of triples $\{(s_{i,t},a_{i,t},r_{i,t})\}_{t=1}^T$, where $T$ is the trajectory horizon. Let us start with the expression of the gradient estimator:
\begin{equation}
	\widehat{\nabla}_{\mathbr{\theta}} J_{\mathcal{M}_{\mathbr{\omega}/\mathbr{\omega}_0}}(\mathbr{\theta}) = \frac{1}{n} \sum_{i=1}^n \sum_{t=0}^{T-1} \gamma^t r_{i,t} 
\underbracket{\left( \frac{\mu_{\mathbr{\omega}}(s_{i,0})}{\mu_{\mathbr{\omega}_0}(s_{i,0})} \prod_{j=0}^t \frac{p_{\mathbr{\omega}}(s_{i,j+1}|s_{i,j},a_{i,j})}{p_{\mathbr{\omega}_0}(s_{i,j+1}|s_{i,j},a_{i,j})} \right)}_{\text{importance weight}}
\sum_{j=0}^t \nabla_{\mathbr{\theta}} \log \pi_{\mathbr{\theta}} \left( a_{i,j} | s_{i,j} \right).
\end{equation}
The expression is obtained starting from the well--known G(PO)MDP gradient estimator and adapting for off--distribution estimation, by introducing the importance weight~\cite{metelli2018policy}. The estimated 2--\Renyi divergence is obtained from the following expression, which represents the empirical second moment of the importance weight:
\begin{equation}
	\widehat{d}_2 (\mathbr{\omega} \| \mathbr{\omega}_0) = \frac{1}{n} \sum_{i=1}^n \left(\frac{\mu_{\mathbr{\omega}}(s_{i,0})}{\mu_{\mathbr{\omega}_0}(s_{i,0})} \prod_{t=1}^{T} \frac{p_{\mathbr{\omega}}(s_{i,t+1}|s_{i,t},a_{i,t})}{p_{\mathbr{\omega}_0}(s_{i,t+1}|s_{i,t},a_{i,t})} \right)^2.
\end{equation}

Refer to~\citet{metelli2018policy} for the theoretical background behind the choice of this objective function. 

In the following, we report the pseudocodes for the environment configuration procedure in the case of application of Identification Rule~\ref{ir:complete} (Algorithm~\ref{alg:ConfCompleteTest}) and Identification Rule~\ref{ir:simplified} (Algorithm~\ref{alg:ConfLinearTest}). 
\begin{algorithm}[H]
    \captionof{algorithm}{Identification Rule~\ref{ir:complete} (Combinatorial) with Environment Configuration.}
    \label{alg:ConfCompleteTest}
    \small
    \textbf{input}: parameter space $\Theta$, configuration space $\Omega$, critical values $c$, number of configuration attempts $N_{\mathrm{conf}}$
    \begin{algorithmic} 
    	\State Initialize $\mathbr{\omega}_0$ arbitrarily
    	\State Collect $\mathcal{D}_0$ observing $\pi^*_0$ in environment $\mathcal{M}_{\mathbr{\omega}_0}$
    	\State Run the Identification Rule~\ref{ir:complete} on $\mathcal{D}_0$ with $\delta'$ and get $\widehat{\mathcal{I}}_0$
    	\State $\widehat{\mathcal{I}} \leftarrow \widehat{\mathcal{I}}_{0}$
    	\For{$I \subseteq \interval : I \notin \widehat{\mathcal{I}}$}
    		\State $\mathbr{\omega}_{i,0} \leftarrow \mathbr{\omega}_0$
    		\State $\mathcal{D}_{i,0} \leftarrow \mathcal{D}$
    		\For{$j = 1,...,N_{\mathrm{conf}}$}
    			\State Optimize $\mathcal{C}_{I}(\mathbr{\omega}/\mathbr{\omega}_{i,j-1})$ getting $\mathbr{\omega}_{i,j}$
    			\State Collect $\mathcal{D}_{i,j}$ observing $\pi^*_{i,j}$ in environment $\mathcal{M}_{\mathbr{\omega}_{i,j}}$
    			\State Run the Identification Rule~\ref{ir:complete} on $\mathcal{D}_{i,j}$ and obtain $\widehat{\mathcal{I}}_{i,j}$
    			\State $\widehat{\mathcal{I}} \leftarrow \widehat{\mathcal{I}} \cup \widehat{\mathcal{I}}_{i,j}$
    		\EndFor
    	\EndFor
        \State \textbf{return} $\widehat{\mathcal{I}}$
    \end{algorithmic}
\end{algorithm}
\begin{algorithm}[H]
    \captionof{algorithm}{Identification Rule~\ref{ir:simplified} (Simplified) with Environment Configuration.}
    \label{alg:ConfLinearTest}
    \small
    \textbf{input}: parameter space $\Theta$, configuration space $\Omega$, critical value $c$, number of configuration attempts $N_{\mathrm{conf}}$
    \begin{algorithmic} 
    	\State Initialize $\mathbr{\omega}_0$ arbitrarily
    	\State Collect $\mathcal{D}_0$ observing $\pi^*_0$ in environment $\mathcal{M}_{\mathbr{\omega}_0}$
    	\State Run the Identification Rule~\ref{ir:simplified} on $\mathcal{D}_0$ and obtain $\widehat{{I}}_0$
    	\State $\widehat{{I}} \leftarrow \widehat{{I}}_{0}$
    	\For{$i \in \interval : i \notin \widehat{{I}}$}
    		\State $\mathbr{\omega}_{i,0} \leftarrow \mathbr{\omega}_0$
    		\State $\mathcal{D}_{i,0} \leftarrow \mathcal{D}$
    		\For{$j = 1,...,N_{\mathrm{conf}}$}
    			\State Optimize $\mathcal{C}_{\{i\}}(\mathbr{\omega}/\mathbr{\omega}_{i,j-1})$ getting $\mathbr{\omega}_{i,j}$
    			\State Collect $\mathcal{D}_{i,j}$ observing $\pi^*_{i,j}$ in environment $\mathcal{M}_{\mathbr{\omega}_{i,j}}$
    			\State Run the Identification Rule~\ref{ir:simplified} on $\mathcal{D}_{i,j}$ and obtain $\widehat{{I}}_{i,j}$
    			\State $\widehat{{I}} \leftarrow \widehat{{I}} \cup \widehat{{I}}_{i,j}$
    		\EndFor
    	\EndFor
        \State \textbf{return} $\{\widehat{I}\}$
    \end{algorithmic}
\end{algorithm}

\section{Additional Experimental details}
In this appendix, we report the full experimental results, along with the hyperparameters employed.

\subsection{Discrete Grid World}\label{apx:discreteGridworld}

\subsubsection{Hyperparameters}
In the following, we report the hyperparameters used for the experiments on the discrete grid world:
\begin{itemize}
\item Horizon ($T$): 50
\item Discount factor ($\gamma$): 0.98
\item Learning steps with G(PO)MDP: 200
\item Batch size: 250
\item Max-likelihood maximum update steps: 1000
\item Max-likelihood learning rate (using Adam): 0.03
\item Number of configuration attempts per feature ($N_{\mathrm{conf}}$): 3
\item Environment configuration update steps: 150
\item Regularization parameter of the \Renyi divergence ($\zeta$): 0.125
\item Significance of the likelihood-ratio tests ($\delta$): 0.01
\end{itemize}

\subsubsection{Example of configuration and identification in the discrete grid world}
In Figure~\ref{fig:gridworld_graphical_conf}, we show a graphical representation of a single experiment with the grid world environment using its configurability to better identify the policy space. The colors inside the squares indicate the probability mass function associated to the initial state distribution, consisting of the agent's position (blue) and the goal position (red), where sharper colors mean higher probabilities. The colored lines represent the features the agent has access to, they are binary features indicating if the agent is on a certain row or column (blue lines) and if the goal is on a certain row or column (red lines). Note that, to avoid redundancy of representation (and so enforcing the identifiability), the last row and column are not explicitly encoded, but they can be represented by the absence of the other rows and columns. When a line is not shown anymore, it means that it has been rejected, i.e., we think the agent has access to that feature. The agent has access to every feature except for the goal columns, i.e., only to its own position and to the goal row are known.

The configuration of the environment is updated in the images at even position, the identification step is performed at even positions. The environment is configured in order to maximize the influence on the gradient of the first -- not rejected -- feature, considering the blue features first and then the red ones. After the model was configured three times for a feature, and the feature has not been rejected, the model was configured for the next one.

We can see that the general trend of this configuration is to change the parameters in order to spread the initial value of the mass probability functions across a greater number of grid cells. This is an expected behavior since with the initial model configuration, very often an episode starts with the agent in the bottom-left of the grid and the goal in the bottom-right, causing the policy to depend mostly on the position of the agent. In fact, only blue column features are rejected at the first iteration, as we can see in the third image. Instead, distributing the probabilities across the whole grid let an episode starts with the two positions extracted almost uniformly. Eventually, the correct policy space is identified. It is interesting to observe that such is can hardly be obtained without the configuration of the environment, given the initial state distribution shown in the first image.

\newcommand{\mybox}[2]{
	\begin{tikzpicture}
	{%
		\setlength{\fboxsep}{0pt}%
		\setlength{\fboxrule}{0.5pt}%
    	\draw (0, 0) node[inner sep=0] {\fbox{\includegraphics[width=0.21\textwidth]{#1}}};
    	\draw (0, -2.2) node {#2};
	}
	\end{tikzpicture}
}

\begin{figure}[h!]
\raggedleft
{%
\setlength{\fboxsep}{0pt}%
\setlength{\fboxrule}{0.5pt}%
\mybox{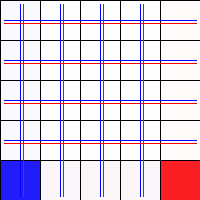}{configuration}\hfill
\mybox{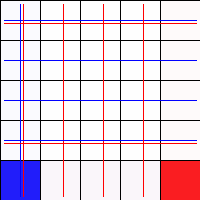}{identification}\hfill
\mybox{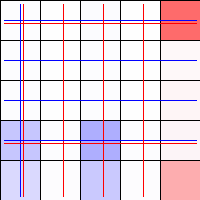}{configuration}\hfill
\mybox{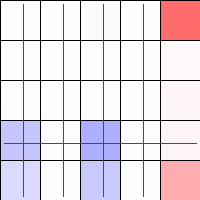}{identification}\hfill
\vspace{0.05\textwidth}
\mybox{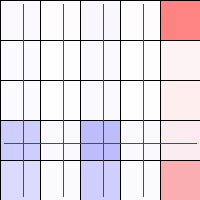}{configuration}\hfill
\mybox{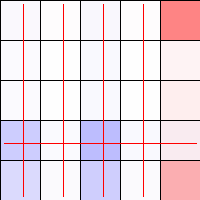}{identification}\hfill
\mybox{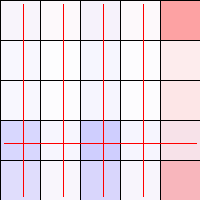}{configuration}\hfill
\mybox{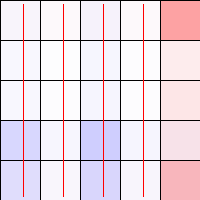}{identification}\hfill
\vspace{0.05\textwidth}
\mybox{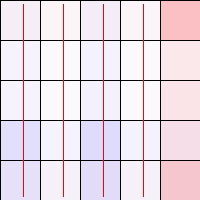}{configuration}\hfill
\mybox{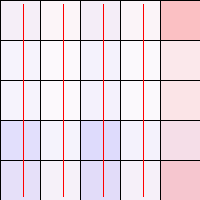}{identification}\hfill
\mybox{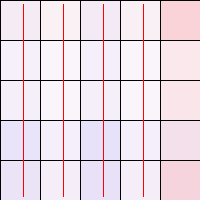}{configuration}\hfill
\mybox{plots/exp3_gridworld/stateImage5A.png}{identification}\hfill
\vphantom{\fbox{\includegraphics[width=0.21\textwidth]{plots/exp3_gridworld/stateImage5A.png}}\hfill}
}%
\caption{Example of configuration and identification in the discrete grid world.}
\label{fig:gridworld_graphical_conf}
\end{figure}

\subsection{Continuous Grid World}\label{apx:continuousGridworld}
In this appendix, we report the experiments performed on the continuous version of the grid world. In this environment, the agent has to reach a goal region, delimited by a circle, starting from an initial position. Both initial position and center of the goal are sampled at the beginning of the episode from a Gaussian distribution with fixed covariance $\mu_{\mathbr{\omega}}$. The supervisor is allowed to change, via the parameters $\mathbr{\omega}$, the mean of this distribution. The agent specifies, at each time step, the speed in the vertical and horizontal direction, by means of a bivariate Gaussian policy with fixed covariance, linear in a set of radial basis functions (RBF) for representing both the current position of the agent and the position of the goal (5$\times$5 both for the agent position and the goal). The feature, and consequently the parameters, that the agent can control are randomly selected at the beginning. In Figure~\ref{fig:contGrid}, we show the results of an experiment analogous to that of the discrete grid world, by comparing $\widehat{\alpha}$ and $\widehat{\beta}$ for the case in which we do not perform environment configuration (no-conf) and the case in which the configuration is performed (conf). Once again, we confirm our findings that configuring the environment allows speeding up the identification process by inducing the agent chaining its policy and, as a consequence, revealing which parameters it can actually control.

\begin{figure}[h!]
\centering
\includegraphics[scale=1.5]{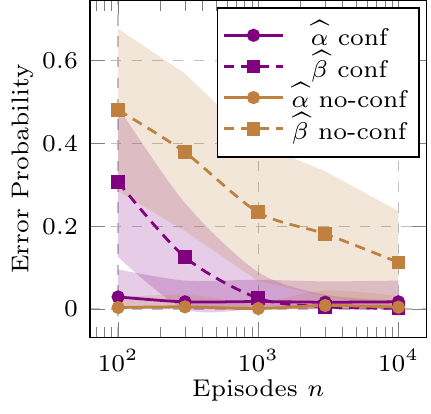}
\caption{$\widehat{\alpha}$ and $\widehat{\beta}$ error for \emph{conf} and \emph{no-conf} cases in the continuous grid world varying the number of episodes $n$. 25 runs 95\% c.i.}\label{fig:contGrid}
\end{figure}

\subsubsection{Hyperparameters}
In the following, we report the hyperparameters used for the experiments on the continuous grid world:
\begin{itemize}
\item Horizon ($T$): 50
\item Discount factor ($\gamma$): 0.98
\item Policy covariance ($\mathbr{\Sigma}$): $0.02^2 \mathbr{I}$
\item Learning steps with G(PO)MDP: 100
\item Batch size: 100
\item Max-likelihood maximum update steps: closed form 
\item Number of configuration attempts per feature ($N_{\mathrm{conf}}$): 3
\item Environment configuration update steps: 100
\item Regularization parameter of the \Renyi divergence ($\zeta$): $1e-6$
\item Significance of the likelihood-ratio tests ($\delta$): 0.01
\end{itemize}

\subsubsection{Example of configuration and identification in the continuous grid world}
In Figure~\ref{fig:cgridworld_graphical_conf}, we show an example of model configuration in the continuous grid world environment. The two filled circles are a graphical representation of the normal distributions from which the initial position of the agent (light blue) and the position of the goal (pink) are sampled at the beginning of each episode. The circumferences correspond to the set of features (RBF) to which the agent has access, among which we want to discover the ones accessible by the agent. Since the policy space is composed by Gaussian policies with mean specified by a linear combination of these features, each one is associated to a parameter. If a circumference is not shown anymore at an iteration step, it means that the hypothesis associated to that feature was rejected, \ie we believe that the agent has access to that feature.

The group of images is an alternated sequence of new environment configurations and parameter identifications. In the first image we can see the initial model with no rejected features. The identification with the initial model yields to the rejection of a certain set of features, which can be seen in the second image. The third image shows the new configuration of the model, in which the mean of the two initial state distributions are moved in order to investigate the remaining features. Then a new test is performed and the result is shown in the fourth image, and so on. In this experiment, the environment was configured in order to maximize the influence of one feature at a time, starting from the blue ones from bottom-left to top-right in row order, and then with the red ones in the same order. Each feature is used to configure the model for a maximum of three times, after that point the next feature is considered.

The only features that were not actually in the agent's set are the red ones on the two top rows. We can see that the mean of the initial position of the agent (a configurable parameter of the environment) always tracked the first available feature yet to be tested, as expected from this experiment. In fact, when the initial position is close enough to those features, the agent often moves around those blue circumferences to reach the goal, making them more important in the definition of the optimal policy. Eventually, the tests reject all the features that are actually accessible by the agent, and only them, yielding to a correct identification of the policy space. The rest of the configurations are not shown, since no more features were rejected. In this experiment, similarly to the discrete grid world case, the use of Conf--MDPs was crucial to obtain this result.

\begin{figure}[h!]
\raggedleft
{%
\setlength{\fboxsep}{0pt}%
\setlength{\fboxrule}{0.5pt}%
\mybox{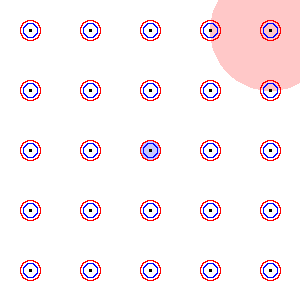}{configuration}\hfill
\mybox{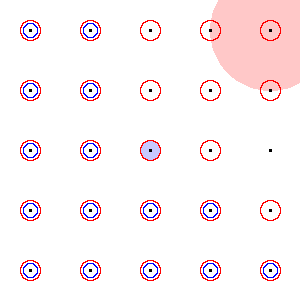}{identification}\hfill
\mybox{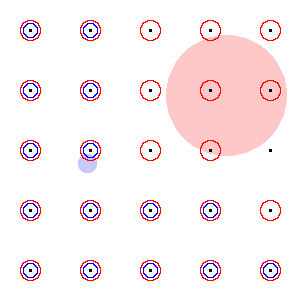}{configuration}\hfill
\mybox{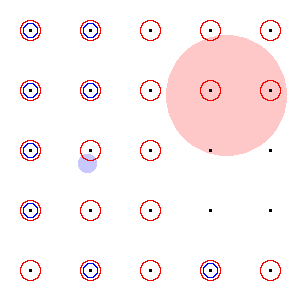}{identification}\hfill
\vspace{0.05\textwidth}
\mybox{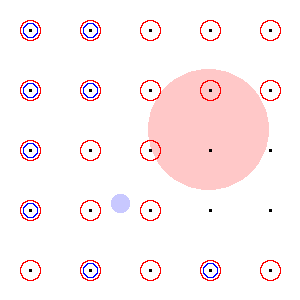}{configuration}\hfill
\mybox{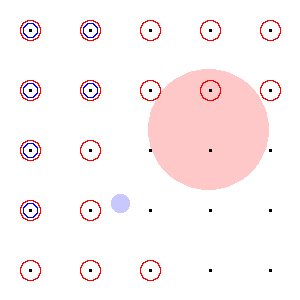}{identification}\hfill
\mybox{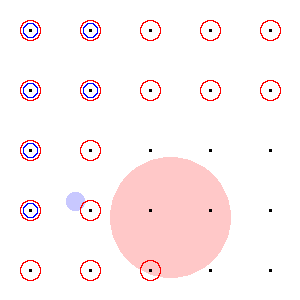}{configuration}\hfill
\mybox{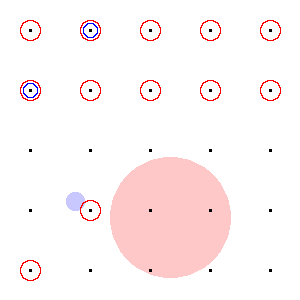}{identification}\hfill
\vspace{0.05\textwidth}
\mybox{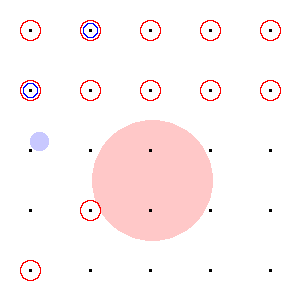}{configuration}\hfill
\mybox{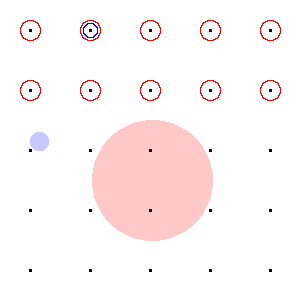}{identification}\hfill
\mybox{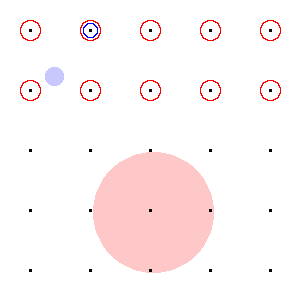}{configuration}\hfill
\mybox{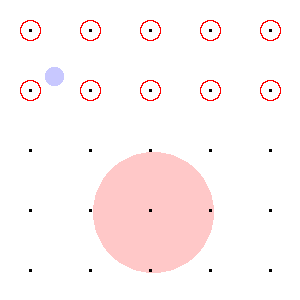}{identification}\hfill
\vspace{0.05\textwidth}
\mybox{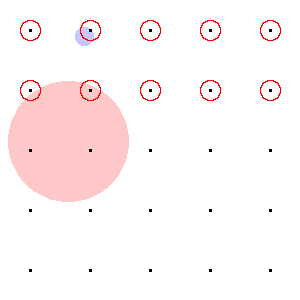}{configuration}\hfill
\mybox{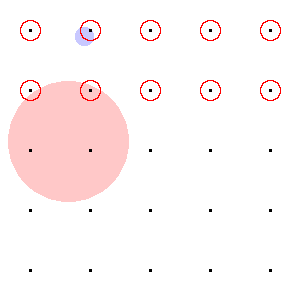}{identification}\hfill
\mybox{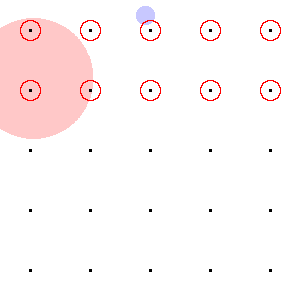}{configuration}\hfill
\mybox{plots/exp1_cgridworld/stateImage7A.png}{identification}\hfill
\vphantom{\fbox{\includegraphics[width=0.21\textwidth]{plots/exp1_cgridworld/stateImage6A.png}}\hfill}
}%
\caption{Example of configuration and identification in the continuous grid world.}
\label{fig:cgridworld_graphical_conf}
\end{figure}

\clearpage
\subsection{Minigolf}\label{apx:experimentsMinigolf}
In the minigolf experiment, the polynomial features obtained from the distance from the goal $x$ and the friction $f$ are the following:
\begin{equation*}
	\mathbr{\phi}(x,f) = \left(1,\, x,\, f,\, \sqrt{x},\, \sqrt{f},\, \sqrt{xf}\right)^T.
\end{equation*}
While agent $\mathscr{A}_1$ perceives all the features, agent $\mathscr{A}_2$ has access to $ \left(1,\, x,\, \sqrt{x} \right)^T$ only.

\subsubsection{Hyperparameters}
In the following, we report the hyperparameters used for the experiments on the minigolf:
\begin{itemize}
\item Horizon ($T$): 20
\item Discount factor ($\gamma$): 0.99
\item Policy covariance ($\mathbr{\Sigma}$): $0.01$
\item Learning steps with G(PO)MDP: 100
\item Batch size: 100
\item Max-likelihood maximum update steps: closed form 
\item Number of configuration attempts per feature ($N_{\mathrm{conf}}$): 10
\item Environment configuration update steps: 100
\item Regularization parameter of the \Renyi divergence ($\zeta$): 0.25
\item Significance of the likelihood-ratio tests ($\delta$): 0.01
\end{itemize}

\subsubsection{Experiment with randomly chosen features}
In the following, we report an additional experiment in the minigolf domain in which the features that the agent can perceive are randomly selected at the beginning, comparing the case in which we do not configure the environment and the case in which environment configuration is performed, and for different number of episodes collected. Although, less visible \wrt to the grid world case, we can see that for some features (\eg $\sqrt{x}$ and $\sqrt{xf}$) the environment configurability is beneficial.

\begin{figure}[h!]
\centering
\subcaptionbox{Without configuration}{\includegraphics[scale=1]{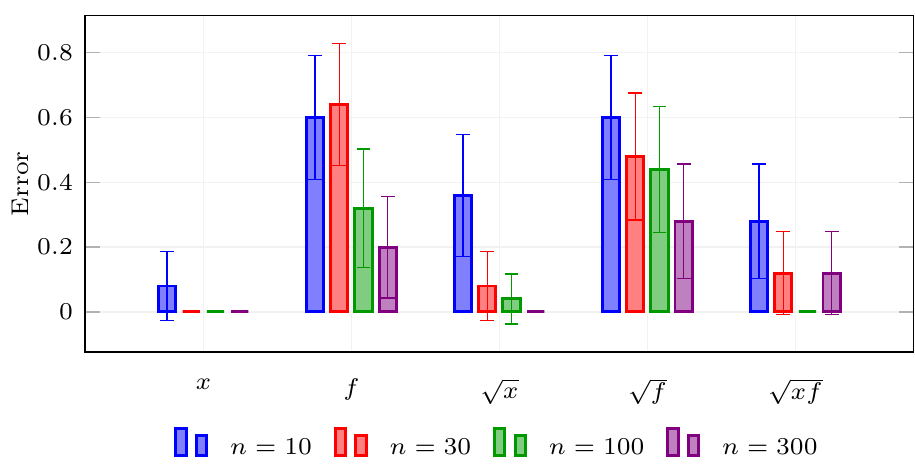}}%
\hfill
\subcaptionbox{With configuration}{\includegraphics[scale=1]{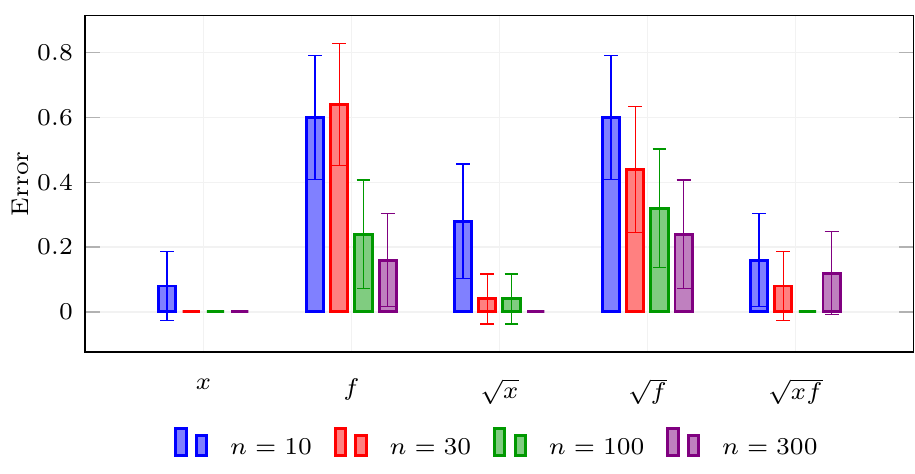}}%
\caption{Experiment with randomly chosen features on the minigolf domain for different number of episodes $n$. 100 runs, 95\% c.i.}
\end{figure}

\subsection{Simulated Car Driving}\label{apx:experimentsCar}
In this environment, an agent has to drive a car to reach the end of the track without running off the road. The control directives are the acceleration and the steering, and are expressed through a two dimensional bounded action space. The car has four sensors oriented in different directions: $-\frac{\pi}{4}$, $-\frac{\pi}{6}$, $\frac{\pi}{6}$, $\frac{\pi}{4}$ \wrt the axis pointing toward the front of the car. The values of these sensors are the normalized distances from the car to the nearest road margin along the direction of the sensor, or the maximum value if the margin is outside the range of the sensor. The complete set of state features is made up by the normalized car speed and the values of the four sensors. In the experiments, the agent has access to the speed and the sensor at angles $\frac{\pi}{6}$ and $\frac{\pi}{4}$. The track consists in a single road segment with a fixed curvature. 
The rewards are given proportionally to the speed of the car, \ie greater speeds yield higher rewards. The episode finishes when the car goes outside the road, and a negative reward is given in this case, when the track is completed, or when a maximum number of time steps is elapsed.

\subsubsection{Hyperparameters}
In the following, we report the hyperparameters used for the experiments on the simulated car driving:
\begin{itemize}
\item Horizon ($T$): 250
\item Discount factor ($\gamma$): 0.996
\item Policy covariance ($\mathbr{\Sigma}$): $0.1 \mathbr{I}$
\item Learning steps with G(PO)MDP: 100
\item Batch size: 50
\item Max-likelihood maximum update steps: 200
\item Max-likelihood learning rate (using Adam): 0.1
\item Significance of the likelihood-ratio tests ($\delta$): 0.1 rescaled by $0.1/5$ for the simplified identification rule and $0.1/32$ for the combinatorial identification rule
\end{itemize}

\end{document}